\documentclass[11pt]{article}

\usepackage[T1]{fontenc}
\usepackage{amssymb}
\usepackage{etoolbox}
\PassOptionsToPackage{table}{xcolor}

\usepackage{notation}
\usepackage{arxiv}

\newcommand{\bema}{\bm{\mathsf{ema}}}

\usepackage{enumitem}
\usepackage{microtype}
\usepackage{float}
\usepackage{tikz}
\usepackage[noEnd=true,indLines=true]{algpseudocodex}

\floatstyle{ruled}
\newfloat{algorithm}{tbp}{loa}
\floatname{algorithm}{Algorithm}
\floatstyle{plain}
\crefname{algorithm}{algorithm}{algorithms}
\Crefname{algorithm}{Algorithm}{Algorithms}
\crefname{lemmainner}{lemma}{lemmas}
\Crefname{lemmainner}{Lemma}{Lemmas}
\crefname{propositioninner}{proposition}{propositions}
\Crefname{propositioninner}{Proposition}{Propositions}

\allowdisplaybreaks
\setlist{leftmargin=*,itemsep=0.25em,topsep=0.4em}

\AtBeginDocument{%
}

\title{Constant Individual Regret in General Games}
\author{Mingyang Liu$^1$, Gabriele Farina$^1$, Asuman Ozdaglar$^1$\\
    $^1$ LIDS, EECS, Massachusetts Institute of Technology\\
    $^1$ \texttt{\{liumy19,gfarina,asuman\}@mit.edu}
}
\date{}

\begin{document}
\maketitle

\begin{abstract}
Uncoupled no-regret dynamics provide a decentralized route to equilibrium, but prior guarantees for individual regret retain a polylogarithmic dependence on the horizon. We remove this dependence for every finite $N$-player normal-form game under full-information feedback. We introduce \emph{ECHO-OFTRL}: optimistic follow-the-regularized-leader (OFTRL) equipped with an EMA cascade for high-order optimism (ECHO), where EMA denotes exponential moving average. The algorithm is deterministic and fully uncoupled. If $m_{\max}$ denotes the largest action-set size, then, simultaneously for every horizon $T\geq1$, it guarantees that each of the $N$ players in the game incurs regret upper bounded by $O(\textrm{poly}(N, \log m_{\max}))$. Our algorithm leverages a new form of optimism inspired by modern filter design.
\end{abstract}

\section{Introduction}

Uncoupled no-regret dynamics model decentralized learning: each player updates
from its own full payoff vector without knowing the other players' objectives
\citep{hart2000simple-regret-matching-RM,greenwald2003general-phi-regret}.
If every player has small external regret, empirical play approaches the set
of coarse correlated equilibria (CCE)
\citep{moulin1978strategically-CCE-def,cesa2006prediction-Hedge}. Regret-minimization methods also underlie several landmark systems for solving
large imperfect-information games
\citep{bowling2015heads-texas-limit,moravvcik2017deepstack,
brown2018superhuman-libratus,brown2019superhuman-pluribus}.

Traditional online-learning analysis treats the payoff sequence as an
arbitrary, possibly adversarial, exogenous process.  Against such a sequence, $\sqrt T$ regret can be easily shown to be unavoidable
\citep{hazan2016introduction}.  In self-play,
however, the payoff sequences
are coupled through a fixed game and the learners' updates. While nonstationary, the smooth update of strategies makes the dynamic landscape faced by each player in the game slightly predictable. Optimistic
methods seek to exploit this predictability to accelerate learning and convergence to equilibrium.
Optimism has also been studied more broadly as a route to acceleration in
optimization and game solving
\citep{NEURIPS2018_e06f967f-optimism-related-work-2,
DBLP:conf/iclr/MertikopoulosLZ19-optMD-new,
farina2019stable-predictive-CFR,piliouras2022fast}.
Fast rates under robustness and nonparametric structure have also been studied
\citep{foster2016learning,daskalakis2022fast}.

The quest for accelerated learning dynamics has driven a long progression of faster rates.
Predictable-sequence methods brought optimism to online learning
\citep{chiang2012online-optimism-related-work-1,
DBLP:conf/nips/RakhlinS13-OOMD}.  The RVU framework
\citep{syrgkanis2015fast-RVU} then gave $T^{1/4}$ individual regret and, more crucially, exposed a key mechanism of analysis that has since been a key mainstay. Unfortunately, that analysis alone is capable of guaranteeing constant \emph{sum} of regrets, but not constant individual regret, as negative regrets can hide another
player's large positive regret.

Later, 
\citet*{chen2020hedging} leveraged the RVU bound, but was able to break through the $T^{1/4}$ barrier in the special case of two-player games. Breaking past the polynomial dependence in $T$, the high-order, discrete-time frequency analysis of Optimistic Hedge of
\citet*{daskalakis2021near-optimal} reduced
individual regret to $O(\log^4T)$.  This approach was subsequently extended to
certain structured polyhedral games \citep{farina2022kernelized}.  Lifted OFTRL then obtained
$O(\log T)$ regret over general convex games by turning ordinary regret into
a nonnegative lifted regret \citep{farina2022near}.  Dynamic learning-rate
control matched this logarithmic horizon dependence while improving the
action dependence to polylogarithmic \citep{soleymani2025faster}. This mechanism proved extremely general and robust across different choices of the convex action set and regularizer \citep{soleymani2025cautious}.  Related work also obtains logarithmic swap
regret in multiplayer games \citep{anagnostides2022uncoupled}.
Finally, Clairvoyant MWU \citep{piliouras2022beyond} obtained a constant regret bound from a fixed-point update that evaluates payoff vectors at the strategy profile being computed. This guarantee is not directly comparable to those above: the
method requires unrolling joint fixed-point computations, and its regret bound
applies only to a sparse subsequence rather than the full history of play.

We summarize the comparable results in \cref{tab:individual-regret-comparison}, where $N$ denotes the number of players and
$m\coloneqq m_{\max}$ denotes the largest number of actions of any player.  All bounds
are specialized to utilities in $[0,1]$; the OFTRL/OOMD row uses entropy
regularization. This history points to a natural question:
\begin{quote}
    \itshape Can uncoupled
regularized learning keep every player's positive regret uniformly bounded in time?
\end{quote}
We show that it can for full-information self-play in finite normal-form games, resolving the horizon-dependence question in this setting. We also remark that because our dynamics are based on a (lifted) version of Hedge, it is possible to kernelize them, making our method applicable beyond normal-form games to a number of important combinatorial settings, including extensive-form games \citep{farina2022kernelized}.

\begin{table}[t]%
\centering%
\caption{Progress toward horizon-independent individual regret via uncoupled
learning dynamics. %
}%
\label{tab:individual-regret-comparison}%
\small%
\setlength{\tabcolsep}{5pt}%
\renewcommand{\arraystretch}{1.15}%
\rowcolors{2}{black!6}{white}%
\begin{tabular}{
>{\raggedright\arraybackslash}p{0.31\linewidth}
>{\raggedright\arraybackslash}p{0.22\linewidth}
>{\raggedright\arraybackslash}p{0.40\linewidth}}
\toprule
Method & Individual regret & Main technique(s) \\
\midrule
OFTRL/OOMD\newline \citep{syrgkanis2015fast-RVU}
& $O\bigl(\sqrt N\log m\,T^{1/4}\bigr)$
& One-step optimism and RVU bounds \\
Optimistic Hedge\newline \citep{chen2020hedging}
& $O\bigl(\log^{5/6}m\,T^{1/6}\bigr)$ \newline (Two-players only)
& Coupled strategy/payoff path-length bounds \\
Optimistic Hedge\newline \citep{daskalakis2021near-optimal}
& $O\bigl(N\log m\log^4T\bigr)$
& Analysis leverages high-order discrete differences \\
LRL-OFTRL\newline \citep{farina2022near}
& $O(Nm\log T)$
& Lifting, nonnegative RVU bounds, and log regularization \\
DLRC-OMWU / Cautious Optimism\newline
\citep{soleymani2025faster,soleymani2025cautious}
& $O\bigl(N\log^2m\log T\bigr)$
& Dynamic pacing, intrinsic Lipschitzness \\
\textbf{ECHO-OFTRL} \newline (This paper)
& $O\bigl(N^{21}\log^4m\bigr)$
& Transfer-function-designed EMA residual and nonnegative RVU bound \\
\bottomrule
\end{tabular}
\vspace{4mm}
\end{table}

\subsection{Contributions}

The main contribution of the paper is to show the following.

\begin{theorem}[informal main theorem]
For every $N$-player finite normal-form game with utilities in $[0,1]$ and full
payoff-vector feedback, there exist deterministic, uncoupled learning dynamics
such that, when used by all players, the individual regret of the generic player $i$ satisfies
\begin{align*}
 \Reg_i(T) \le \sum_{j=1}^N[\Reg_j(T)]_+
 \leq C(N+1)^{21}\rbr{1+\log(m_{\max}+1)}^4
\end{align*}
 simultaneously for every horizon $T\geq1$,
where $C$ is a universal constant.
\end{theorem}

The player dependence is explicitly polynomial and the bound is independent
of $T$; we made no attempt to optimize either displayed exponent.

Technique-wise, our main methodological contribution is a signal-processing construction of a new,
stable high-order predictor.  Standard optimistic learning uses the preceding
payoff vector to predict the current one, so the prediction error is the
one-step change in payoffs.  Repeating this idea $N$
times would give the raw difference $(I-\mathsf L)^N$, where $\mathsf L$ is
the one-round delay: for a payoff sequence $\bg$ and every integer $t$,
$\rbr{\mathsf L\bg}^{(t)}=\bg^{(t-1)}$.  The coefficients of the raw difference
have total absolute value $2^N$, so this direct construction can amplify
oscillations exponentially.  We instead set
\begin{align*}
 \rho\coloneqq&1-\frac1N,
 \qquad
 \delta\coloneqq\frac{I-\mathsf L}{I-\rho\mathsf L},\\
 \widehat\bg
 \coloneqq&(I-\delta^N)\bg,
 \qquad
 \bg-\widehat\bg
 =\delta^N\bg.
\end{align*}
The filter
$I-\delta=(1-\rho)\mathsf L(I-\rho\mathsf L)^{-1}$ is a one-pole EMA of past
values.  This preconditioning reduces the worst-case frequency gain of the
$N$-fold difference from $2^N$ to $(2/(1+\rho))^N\leq e$, while the
time-domain filter bounds used in the proof remain polynomial in $N$.

In short, high-order optimism lets us
bound the cumulative prediction-error term by a horizon-independent constant
plus strategy movement; the negative movement term absorbs the latter,
yielding constant regret.

\section{Preliminaries}
\label{sec:preliminaries}

For every positive integer $n$, let $[n]\coloneqq\cbr{1,\ldots,n}$.  For a vector
$\bz\in\RR^n$ and $a\in[n]$, write $z(a)$ for its $a^{\rm th}$ coordinate.  For
$1\leq p<\infty$, define
$\nbr{\bz}_p\coloneqq\rbr{\sum_{a=1}^n\abr{z(a)}^p}^{1/p}$, and define
$\nbr{\bz}_\infty\coloneqq\max_{a\in[n]}\abr{z(a)}$.  For a scalar $r\in\RR$, define
$[r]_+\coloneqq\max\cbr{r,0}$.  The $(n-1)$-dimensional probability simplex is
\begin{align*}
 \Delta^n\coloneqq\cbr{\bx\in[0,1]^n:\sum_{a=1}^nx(a)=1}.
\end{align*}
We use $\inner{\cdot}{\cdot}$ for the standard inner product and $\one$ for the
all-ones vector of the appropriate dimension.
For a tuple $\bx\coloneqq(\bx_i)_i$, write $\bx_{-i}\coloneqq(\bx_j)_{j\neq i}$.  Superscripts in
parentheses index time, so $\bz^{(t)}$ denotes the value of a sequence $\bz$
at round $t$.
For a differentiable map $\Phi:E\to F$ between finite-dimensional Euclidean
spaces, $\bx\in E$, and $\bxi\in E$, $\bJ_\Phi(\bx):E\to F$ denotes its
Jacobian at $\bx$, and
$\bJ_\Phi(\bx)[\bxi]\coloneqq
\left.\frac{\ud}{\ud r}\Phi(\bx+r\bxi)\right|_{r=0}$ is its directional derivative
along $\bxi$.  If $E=\prod_{j=1}^kE_j$, then, for $j\in[k]$, write
$\partial_j\Phi$ for the derivative with respect to the $j^{\rm th}$ block.  For
scalar-valued $f$, $\nabla f$ denotes its gradient, so
$\bJ_f(\bx)[\bxi]=\inner{\nabla f(\bx)}{\bxi}$, and $\nabla^2f$ denotes its Hessian.

\subsection{Finite normal-form games and regret}

There are $N\geq2$ players.  Player $i$ has action set $[m_i]$, where
$m_i\geq2$, and chooses a mixed strategy $\bx_i\in\Delta^{m_i}$.  Define the
largest action-set size by
\begin{align*}
 m_{\max}\coloneqq\max_{i\in[N]}m_i.
\end{align*}
Player $i$'s utility function
$\cU_i\colon\prod_{j=1}^N\Delta^{m_j}\to[0,1]$ is the multilinear extension of a
pure-profile payoff function taking values in $[0,1]$.  Define the expected
payoff vector by
\begin{align*}
\cU_i(\bx_{-i})
 \coloneqq\nabla_{\bx_i}\cU_i(\bx)
 \in[0,1]^{m_i}.
\end{align*}
It is multiaffine in the opponents' strategies and does not depend on
$\bx_i$.

For a horizon $T\geq1$, at every round $t\in[T]$, the players choose
$\bx^{(t)}\in\prod_{j=1}^N\Delta^{m_j}$ simultaneously and player
$i$ observes the full vector
$\bu_i^{(t)}\coloneqq\cU_i(\bx_{-i}^{(t)})$.  Its external
regret is
\begin{align*}
 \Reg_i(T)
 \coloneqq&\max_{\bx_i\in\Delta^{m_i}}\sum_{t=1}^T
 \sbr{\cU_i(\bx_i,\bx_{-i}^{(t)})-\cU_i(\bx^{(t)})}
 \\
 =&\max_{\bx_i\in\Delta^{m_i}}\sum_{t=1}^T
 \inner{\bu_i^{(t)}}{\bx_i-\bx_i^{(t)}}.
\end{align*}
The equality follows from multilinearity in player $i$'s strategy.

\section{Lifted optimistic dynamics}
\label{sec:finite-algorithm}

\subsection{Lifted regret}

Following \citet{farina2022near}, for every player $i\in[N]$, define
\begin{align*}
 \widetilde\Delta_i
 \coloneqq\cbr{\rbr{\lambda,\by}:0\leq\lambda\leq1,
          \ \by\in\lambda\Delta^{m_i}}.
\end{align*}
For every $i\in[N]$ and $(\lambda,\by)\in\widetilde\Delta_i$ with
$\lambda>0$, the point $(\lambda,\by)$ plays
$\bx\coloneqq\by/\lambda$.
For every player $i\in[N]$, define $\mathcal G_i$ by
\begin{align*}
 \mathcal G_i(\bx)
 \coloneqq
 \cU_i(\bx_{-i})
 -\inner{\cU_i(\bx_{-i})}{\bx_i}\one.
\end{align*}
For every player $i\in[N]$, define
$\bg_i=\rbr{\bg_i^{(s)}}_{s\in\ZZ}$ by
\begin{align*}
 \bg_i^{(s)}
 \coloneqq
 \begin{cases}
  \mathcal G_i(\bx^{(s)}),&s\geq1,\\
  \boldsymbol 0,&s\leq0,
 \end{cases}
 \qquad s\in\ZZ.
\end{align*}
For every $i\in[N]$ and integer $s\geq1$,
\begin{align*}
 \bg_i^{(s)}
 =&\bu_i^{(s)}-\inner{\bu_i^{(s)}}{\bx_i^{(s)}}\one,
 \qquad
 \nbr{\bg_i^{(s)}}_\infty\leq1,\\
 \inner{\bg_i^{(s)}}{\bx_i^{(s)}}=&0.
\end{align*}

For every $i\in[N]$ and round $t\geq1$, write
$\bmu_i^{(t)}=\rbr{\lambda_i^{(t)},
\lambda_i^{(t)}\bx_i^{(t)}}\in\widetilde\Delta_i$
for player $i$'s lifted strategy at round $t$.
For every $i\in[N]$ and $T\geq1$, write each comparator in
$\widetilde\Delta_i$ as
$\rbr{\widehat\lambda_i,\widehat\lambda_i\widehat\bx_i}$, where
$\widehat\lambda_i\in[0,1]$ and $\widehat\bx_i\in\Delta^{m_i}$.
Since $\inner{\bg_i^{(t)}}{\lambda_i^{(t)}\bx_i^{(t)}}=\lambda_i^{(t)}\inner{\bg_i^{(t)}}{\bx_i^{(t)}}=0$ for every
$t\geq1$,
\begin{align}
 \max_{\rbr{\widehat\lambda_i,
 \widehat\lambda_i\widehat\bx_i}\in\widetilde\Delta_i}
 \sum_{t=1}^T
 \inner{\bg_i^{(t)}}{\widehat\lambda_i\widehat\bx_i-
 \lambda_i^{(t)}\bx_i^{(t)}}
 =\max_{0\leq\widehat\lambda_i\leq1}\widehat\lambda_i\Reg_i(T)
 =[\Reg_i(T)]_+.
 \label{eq:positive-part-lift}
\end{align}

\subsection{The square-root entropy response}

For every integer $m>0$ and $\bx\in\Delta^m$, $\Gamma_m$ bounds
$\operatorname{Var}_{a\sim\bx}[-\log x(a)]$. Formally,
\begin{align*}
 \Gamma_m\coloneqq (\log m)^2+2\log m+2,
 \qquad
 \Xi_m\coloneqq10^4(1+\Gamma_m).
\end{align*}
For every integer $m\geq2$ and $\bx\in\Delta^m$, define
$\psi_m(\bx)\coloneqq\sum_{a=1}^m x(a)\log x(a)$, where
$0\log0\coloneqq0$.  For every integer $m\geq2$, define
\begin{align*}
 \widetilde\Delta^m
 \coloneqq\cbr{\rbr{\lambda,\by}:0\leq\lambda\leq1,
          \ \by\in\lambda\Delta^m}.
\end{align*}
For every integer $m\geq2$, define $\widetilde\psi_m$ on
$\widetilde\Delta^m$ by
\begin{align*}
 \widetilde\psi_m(\lambda,\by)
 \coloneqq-\sqrt{1-\lambda}
 +\sqrt\lambda\sbr{\psi_m\rbr{\by/\lambda}-2\Gamma_m-3},
 \qquad \lambda>0.
\end{align*}
Set $\widetilde\psi_m(0,\boldsymbol 0)\coloneqq-1$ for every integer
$m\geq2$.  For every integer $m\geq2$ and $\btheta\in\RR^m$, define
\begin{align*}
 Q_m(\btheta)
 \coloneqq\argmax_{\rbr{\lambda,\by}\in\widetilde\Delta^m}
 \cbr{\inner\btheta\by-\widetilde\psi_m(\lambda,\by)}.
\end{align*}

The next lemma establishes some basic properties of $Q_m(\btheta)$.
\begin{lemma}
\label{lem:convex-power-entropy}
For every integer $m\geq2$, the function $\widetilde\psi_m$ is finite and
convex on $\widetilde\Delta^m$ and strictly convex in its relative interior.
For every integer $m\geq2$ and $\btheta\in\RR^m$, the maximizer defining
$Q_m(\btheta)$ is unique and belongs to
$\operatorname{relint}\widetilde\Delta^m$.  More explicitly, for every
integer $m\geq2$ and $\btheta\in\RR^m$, the function on $(0,1)$ given by
\begin{align*}
 \lambda'\longmapsto
 \frac{1}{\sqrt{1-\lambda'}}
 -\frac{2\Gamma_m+3
 +\log\rbr{\sum_{a=1}^m\exp\rbr{\sqrt{\lambda'}\,\theta(a)}}}
 {\sqrt{\lambda'}}
 -\frac{\sum_{a=1}^m\theta(a)\exp\rbr{\sqrt{\lambda'}\,\theta(a)}}
 {\sum_{a=1}^m\exp\rbr{\sqrt{\lambda'}\,\theta(a)}}
\end{align*}
is strictly increasing, tends to $-\infty$ at the left endpoint, and tends to
$+\infty$ at the right endpoint.  Let $\lambda\in(0,1)$ denote its unique
zero.  Then
\begin{align*}
 Q_m(\btheta)
 =\rbr{\lambda,
 \frac{\lambda\rbr{\exp\rbr{\sqrt\lambda\,\theta(a)}}_{a\in[m]}}
 {\sum_{b=1}^m\exp\rbr{\sqrt\lambda\,\theta(b)}}}.
\end{align*}
Moreover, for every integer $m\geq2$,
\begin{align*}
 \sup_{\bmu\in\widetilde\Delta^m}\widetilde\psi_m(\bmu)
 -\inf_{\bmu\in\widetilde\Delta^m}\widetilde\psi_m(\bmu)
 \leq2\Gamma_m+\log m+4.
\end{align*}
\end{lemma}

The proof is deferred to \Cref{sec:power-entropy-lift}.

\begin{remark}
Fix an integer $m\geq2$ and $\btheta\in\RR^m$.  By
\Cref{lem:convex-power-entropy}, binary search on the sign of the strictly
increasing function in the lemma localizes its unique zero $\lambda$ and the formula in the lemma determines $Q_m(\btheta)$.
\end{remark}

For every integer $m\geq2$ and $\btheta\in\RR^m$, write
\begin{align*}
 \rbr{\lambda_m(\btheta),\by_m(\btheta)}
 \coloneqq Q_m(\btheta),
 \qquad
 \bx_m(\btheta)
 \coloneqq \by_m(\btheta)/\lambda_m(\btheta).
\end{align*}
For every integer $m\geq2$ and
$\bmu,\bmu'\in\operatorname{relint}\widetilde\Delta^m$, define
\begin{align*}
 D_{\widetilde\psi_m}(\bmu',\bmu)
 \coloneqq&\widetilde\psi_m(\bmu')-\widetilde\psi_m(\bmu)
   -\inner{\nabla\widetilde\psi_m(\bmu)}{\bmu'-\bmu},\\
 D_{\widetilde\psi_m}^{\rm sym}(\bmu',\bmu)
 \coloneqq&D_{\widetilde\psi_m}(\bmu',\bmu)
 +D_{\widetilde\psi_m}(\bmu,\bmu').
\end{align*}

For every $i\in[N]$, abbreviate
\begin{align*}
 \psi_i\coloneqq&\psi_{m_i},
 \qquad
 \widetilde\psi_i\coloneqq \widetilde\psi_{m_i},
 \qquad
 Q_i\coloneqq Q_{m_i},
 \\
 \Gamma_i\coloneqq&\Gamma_{m_i},
 \qquad
 \Xi_i\coloneqq \Xi_{m_i}.
\end{align*}
Set $\Xi_{\max}\coloneqq\max_{i\in[N]}\Xi_i$.  For every $i\in[N]$ and
$\btheta\in\RR^{m_i}$, write
$Q_i(\btheta)=(\lambda_{m_i}(\btheta),
\lambda_{m_i}(\btheta)\bx_{m_i}(\btheta))$.

\subsection{ECHO-OFTRL}

We call the algorithm below \emph{ECHO-OFTRL}, where ECHO stands for
\emph{EMA cascade for high-order optimism} and EMA denotes
\emph{exponential moving average}.

Fix a common learning rate $\eta>0$.  On every space of bounded bi-infinite
sequences in a finite-dimensional normed space, let $I$ be the identity
operator and let $\mathsf L$ be the backward shift, defined for every
sequence $\bz$ in that space and integer $t$ by
$\rbr{\mathsf L\bz}^{(t)}\coloneqq\bz^{(t-1)}$.  Set
\begin{align}
 \rho\coloneqq&1-\frac1N,
 \qquad
 \mathsf D\coloneqq I-\mathsf L,
 \qquad
 \mathsf A\coloneqq (I-\rho\mathsf L)^{-1}
 =\sum_{k=0}^{\infty}\rho^k\mathsf L^k,
 \nonumber\\
 \delta\coloneqq&\mathsf A\mathsf D
 =\frac{I-\mathsf L}{I-\rho\mathsf L}.
 \label{eq:stable-filter-law}
\end{align}
Since $\mathsf A$ and $\mathsf D$ commute, $\delta=\mathsf D\mathsf A$.
Thus, $\delta$ takes the one-step difference of the exponentially
weighted sequence generated by $\mathsf A$.

For every sequence
$\bz$ that vanishes at nonpositive times and every
integer $t\geq1$,
\begin{align*}
 \rbr{(I-\delta)\bz}^{(t)}
 =(1-\rho)\sum_{k=1}^{\infty}\rho^{k-1}\bz^{(t-k)}.
\end{align*}

For every $i\in[N]$, define
\begin{align}
 \widehat\bg_i\coloneqq&(I-\delta^N)\bg_i,
 \nonumber\\
 \be_i\coloneqq&\bg_i-\widehat\bg_i
 =\delta^N\bg_i.
 \label{eq:stable-filter-prediction}
\end{align}
Thus, for every $i\in[N]$, the prediction error $\be_i$ is the stable
$N^{\rm th}$-order filtered difference $\delta^N\bg_i$.  Since
$I-\delta^N=(I-\delta)\sum_{h=1}^N\delta^{h-1}$, the preceding formula for
$I-\delta$ shows that $\widehat\bg_i^{(t)}$ depends only on feedback from
rounds before $t$.  The following EMA cascade computes this predictor
recursively.

Every player $i\in[N]$ maintains the sequences
$(\bema_{i,h}^{(t)})_{t\geq1}$ in $\RR^{m_i}$ for $h\in[N]$,
initialized by
$\bema_{i,h}^{(1)}\coloneqq\boldsymbol 0$ for every $h\in[N]$.
For every $i\in[N]$ and integer $t\geq1$, define
\begin{align*}
 \bS_i^{(t-1)}\coloneqq\sum_{s=1}^{t-1}\bg_i^{(s)}.
\end{align*}
For every $i\in[N]$ and round $t\geq1$, player $i$ forms
$\widehat\bg_i^{(t)}$ and plays
\begin{align}
 \widehat\bg_i^{(t)}
 =&\sum_{h=1}^N\bema_{i,h}^{(t)},
 \label{eq:ema-cascade-prediction}\\
 \bmu_i^{(t)}
 \coloneqq& Q_i\rbr{\eta\rbr{\bS_i^{(t-1)}+\widehat\bg_i^{(t)}}},
 \qquad
 \bmu_i^{(t)}=\rbr{\lambda_i^{(t)},
 \lambda_i^{(t)}\bx_i^{(t)}}.
 \label{eq:lifted-oftrl-update}
\end{align}
For every $i\in[N]$ and $t\geq1$, after observing $\bu_i^{(t)}=\cU_i(\bx_{-i}^{(t)})$ and computing
$\bg_i^{(t)}$, player $i$ updates all $N$ vectors simultaneously by
\begin{align}
 \bema_{i,h}^{(t+1)}
 \coloneqq\rho\bema_{i,h}^{(t)}
 +(1-\rho)\rbr{\bg_i^{(t)}
 -\sum_{\ell=1}^{h-1}\bema_{i,\ell}^{(t)}},
 \qquad h\in[N].
 \label{eq:ema-cascade-update}
\end{align}
\Cref{eq:stable-filter-law} gives
$(I-\rho\mathsf L)(I-\delta)=(1-\rho)\mathsf L$.  Hence, for every bounded
bi-infinite sequence $\bz$ in a finite-dimensional normed space and every
integer $t\geq0$,
$\rbr{(I-\delta)\bz}^{(t+1)}
=\rho\rbr{(I-\delta)\bz}^{(t)}+(1-\rho)\bz^{(t)}$.
Fix $i,h\in[N]$ and an integer $t\geq1$.  If
$\bema_{i,\ell}^{(t)}
=\rbr{(I-\delta)\delta^{\ell-1}\bg_i}^{(t)}$ for every $\ell<h$, then
\begin{align*}
 \bg_i^{(t)}-\sum_{\ell=1}^{h-1}\bema_{i,\ell}^{(t)}
 =\rbr{\rbr{I-(I-\delta)
 \sum_{\ell=1}^{h-1}\delta^{\ell-1}}\bg_i}^{(t)}
 =\rbr{\delta^{h-1}\bg_i}^{(t)},
\end{align*}
where the last equality is the geometric-series identity.  The zero initial
states and \Cref{eq:ema-cascade-update} therefore imply, by induction
over $h\in[N]$, that
$\bema_{i,h}^{(t)}
=\rbr{(I-\delta)\delta^{h-1}\bg_i}^{(t)}$ for every integer $t\geq1$.
Summing over $h\in[N]$ and using
$(I-\delta)\sum_{h=1}^N\delta^{h-1}=I-\delta^N$ verifies
\Cref{eq:ema-cascade-prediction}.

For every $i\in[N]$, extend $\bx_i$ and $\lambda_i$ by
$\bx_i^{(t)}\coloneqq\bx_{m_i}(\boldsymbol 0)$ and
$\lambda_i^{(t)}\coloneqq\lambda_{m_i}(\boldsymbol 0)$ for $t\leq0$, and write
$\bx_i\coloneqq(\bx_i^{(t)})_{t\in\ZZ}$ and
$\lambda_i\coloneqq(\lambda_i^{(t)})_{t\in\ZZ}$.
Thus, for every $i\in[N]$ and $t\geq1$, player $i$ forms
$\widehat\bg_i^{(t)}$ from past
feedback, computes and plays $\bx_i^{(t)}$, observes $\bu_i^{(t)}$, computes
$\bg_i^{(t)}$, evaluates $\be_i^{(t)}$, updates the cascade by
\Cref{eq:ema-cascade-update}, and sets
$\bS_i^{(t)}\coloneqq\bS_i^{(t-1)}+\bg_i^{(t)}$.  For every $i\in[N]$, the
cascade uses $O(Nm_i)$ memory and arithmetic per round.
For every $i\in[N]$ and $t\geq1$, \Cref{lem:convex-power-entropy} gives
$0<\lambda_i^{(t)}<1$, so $\bx_i^{(t)}$ is well defined.

For every $i\in[N]$ and $t\geq1$, fixing $\lambda_i^{(t)}$ in the lifted
optimistic-FTRL objective gives, for every $a\in[m_i]$,
\begin{align*}
 x_i^{(t)}(a)
 =\frac{\exp\rbr{
   \eta\sqrt{\lambda_i^{(t)}}
   [\bS_i^{(t-1)}+\widehat\bg_i^{(t)}](a)}}
 {\sum_{b=1}^{m_i}\exp\rbr{
   \eta\sqrt{\lambda_i^{(t)}}
   [\bS_i^{(t-1)}+\widehat\bg_i^{(t)}](b)}}.
\end{align*}
Thus, for every $i\in[N]$ and $t\geq1$, $\bx_i^{(t)}$ is the Hedge strategy with learning rate
$\eta\sqrt{\lambda_i^{(t)}}$.

For every $i\in[N]$ and $t\geq1$, \Cref{alg:echo-oftrl} computes
$\be_i^{(t)}$ in
\Cref{eq:stable-filter-prediction} with one loop over $h\in[N]$.
\begin{algorithm}[H]
\caption{ECHO-OFTRL for player $i$}
\label{alg:echo-oftrl}
\begin{algorithmic}[1]
\Require Common learning rate $\eta>0$ and $\rho=1-1/N$
\State Initialize $\bS_i^{(0)}\gets\boldsymbol 0$ and
       $\bema_{i,h}^{(1)}\gets\boldsymbol 0$ for every $h\in[N]$
\For{$t=1,2,\dots$}
  \State Set
         $\widehat\bg_i^{(t)}
         \gets\sum_{h=1}^N\bema_{i,h}^{(t)}$
  \State Compute $\bmu_i^{(t)}$ and play $\bx_i^{(t)}$ using
         \Cref{eq:lifted-oftrl-update}
  \State Observe $\bu_i^{(t)}$, compute
         $\bg_i^{(t)}\gets\bu_i^{(t)}
         -\inner{\bu_i^{(t)}}{\bx_i^{(t)}}\one$, and set
         $\be_{i,0}^{(t)}\gets\bg_i^{(t)}$
  \For{$h=1,\dots,N$}
    \State Set
           $\bema_{i,h}^{(t+1)}
           \gets\rho\bema_{i,h}^{(t)}
           +(1-\rho)\be_{i,h-1}^{(t)}$
    \State Set
           $\be_{i,h}^{(t)}
           \gets\be_{i,h-1}^{(t)}-\bema_{i,h}^{(t)}$
  \EndFor
  \State Set $\be_i^{(t)}\gets\be_{i,N}^{(t)}$ and
         $\bS_i^{(t)}
         \gets\bS_i^{(t-1)}+\bg_i^{(t)}$
\EndFor
\end{algorithmic}
\end{algorithm}

\subsection{Main result}

\begin{theorem}
\label{thm:entropy-main}
In every $N$-player finite normal-form game with utilities in $[0,1]$, suppose
every player $i\in[N]$ runs ECHO-OFTRL as specified in
\Cref{alg:echo-oftrl} with $Q_i$ and the
constant learning rate
\begin{align*}
 \eta\coloneqq\frac{2^{-94}}{N^{20}\Xi_{\max}}.
\end{align*}
Then, simultaneously for every horizon $T\geq1$,
\begin{align*}
 \sum_{i=1}^N[\Reg_i(T)]_+
 \leq2^{112}N^{21}
 \rbr{1+\log\rbr{m_{\max}+1}}^4.
\end{align*}
\end{theorem}

A proof sketch is in \Cref{sec:multilinear-analysis} and the formal proof is in
\Cref{sec:entropy-main-proof}.

\section{Proof sketch}
\label{sec:multilinear-analysis}

The formal proof of \Cref{thm:entropy-main} appears in
\Cref{sec:entropy-main-proof}; the subsequent appendices prove its
geometric and filter estimates.  This section outlines the main steps of the
argument; $O\rbr{\cdot}$ suppresses only universal numerical factors.

\subsection{The optimistic-FTRL bound}

Fix a horizon $T\geq1$.  For every player $i\in[N]$, set
$\bmu_i^{(0)}\coloneqq Q_i(\boldsymbol 0)$ and define
\begin{align*}
 \mathcal P_{i,T}
 \coloneqq\sum_{t=1}^T
 D_{\widetilde\psi_i}(\bmu_i^{(t)},\bmu_i^{(t-1)}),
 \qquad
 \mathcal P_T^\Xi\coloneqq \sum_{i=1}^N\Xi_i\mathcal P_{i,T}.
\end{align*}
The optimistic-FTRL calculation in \Cref{eq:lifted-rvu}, together
with \Cref{eq:stable-filter-prediction}, gives
\begin{align*}
 \sum_{i=1}^N[\Reg_i(T)]_+
 \leq\frac4\eta\sum_{i=1}^N\Gamma_i
 +80\eta\sum_{i=1}^N\sum_{t=1}^T\Xi_i
  \rbr{\lambda_i^{(t)}}^{3/2}
  \nbr{\rbr{\delta^N\bg_i}^{(t)}}_\infty^2
 -\frac1{2\eta}\sum_{i=1}^N\mathcal P_{i,T}.
\end{align*}
It therefore remains to prove the variation estimate
\begin{align*}
 \sum_{i=1}^N\sum_{t=1}^T\rbr{\lambda_i^{(t)}}^{3/2}
 \nbr{\rbr{\delta^N\bg_i}^{(t)}}_\infty^2
 \leq O\rbr{N^{16}\rbr{1+\mathcal P_T^\Xi}}.
\end{align*}
The weights $(\eta\Xi_i)_{i\in[N]}$ are inserted only after this estimate,
when the resulting sum is absorbed by
$-\sum_i\mathcal P_{i,T}/(2\eta)$.

\subsection{A recursive bound}

For a positive weight sequence $w=\rbr{w^{(t)}}_{t\leq T}$, a player $i\in[N]$,
and an integer $1\leq h\leq N$, define
\begin{align*}
 \cE_i^h(w)
 \coloneqq\sum_{t=1}^Tw^{(t)}
 \nbr{\rbr{\delta^h\bx_i}^{(t)}}_1^2.
\end{align*}
By \Cref{lem:filtered-multiaffine-expansion}, for every $i\in[N]$,
integer $1\leq h\leq N$, and positive weight sequence
$w=\rbr{w^{(t)}}_{t\leq T}$ satisfying
$w^{(t)}\leq1$ for every $t\in[T]$,
\begin{align*}
 \sum_{t=1}^Tw^{(t)}
 \nbr{\rbr{\delta^h\bg_i}^{(t)}}_\infty^2
 \leq8N\sum_{j=1}^N\cE_j^h(w)
 +O\rbr{N^{13}+\eta^2\Xi_{\max}^2N^{27}\mathcal P_T^\Xi}.
\end{align*}

For every $i\in[N]$, integer $1\leq h\leq N$, and positive weight sequence
$w=\rbr{w^{(t)}}_{t\leq T}$ satisfying
$w^{(t)}\leq\sqrt{\lambda_i^{(t)}}$ for every $t\in[T]$,
\Cref{lem:filtered-action-path-bound} gives
\begin{align}
 \cE_i^h(w)
 \leq O\rbr{N^{12}\Xi_i\mathcal P_{i,T}}.\label{eq:direct-path-bound}
\end{align}
For every $r\in[N]$, $S\subseteq[N]\setminus\cbr r$, and $t\in\ZZ$,
set
\begin{align}
 w_{r,S}^{(t)}
 \coloneqq\rbr{\lambda_r^{(t)}}^{3/2}
 \prod_{k\in S}\lambda_k^{(t)}.
 \label{eq:subset-weight}
\end{align}
Fix $r,j\in[N]$ and $S\subseteq[N]\setminus\cbr r$, and set
$h=N-|S|$.  If $j\in S\cup\cbr r$, then
$w_{r,S}^{(t)}\leq\sqrt{\lambda_j^{(t)}}$ for every $t\in[T]$, so
\Cref{eq:direct-path-bound} applies.  If $j\notin S\cup\cbr r$, then
$h\geq2$ and
\begin{align*}
 \mathsf A\delta^{h-1}(I-\mathsf D\delta^N)
 =\rbr{\mathsf A-\delta^{N+1}}\delta^{h-1}.
\end{align*}
The proof of \Cref{eq:finite-label-transfer} verifies the hypothesis of
\Cref{lem:slow-weight-convolution} for every weight in
\Cref{eq:subset-weight}, in particular $w_{r,S\cup\cbr j}$.  Combining this with
\Cref{lem:response-secant-estimates,lem:filtered-response-identity,lem:filtered-multiaffine-expansion}
gives
\begin{align*}
 \cE_j^{N-|S|}(w_{r,S})
 \leq O\rbr{N^{13}\eta^2\Xi_{\max}^2
 \sum_{k=1}^N\cE_k^{N-\abr{S\cup\cbr j}}
 \rbr{w_{r,S\cup\cbr j}}
 +N^{12}\rbr{1+\mathcal P_T^\Xi}}.
\end{align*}
The constants are specified in
\Cref{eq:finite-label-transfer}.

\subsection{Iterating the bound}

For every $r\in[N]$, begin with $S=\varnothing$, for which
$w_{r,\varnothing}^{(t)}=(\lambda_r^{(t)})^{3/2}$ for every $t\in\ZZ$.
Consider a current term $\cE_j^{N-|S|}(w_{r,S})$, where $r,j\in[N]$ and
$S\subseteq[N]\setminus\cbr r$.  If $j\notin S\cup\cbr r$, the preceding
estimate replaces $S$ by $S\cup\cbr j$ and lowers the order by one; otherwise,
\Cref{eq:direct-path-bound} applies.  Thus at most $N-1$ additions precede a
term to which \Cref{eq:direct-path-bound} applies.

Combining the resulting estimates gives
\begin{align*}
 \sum_{i=1}^N\sum_{t=1}^T\rbr{\lambda_i^{(t)}}^{3/2}
 \nbr{\rbr{\delta^N\bg_i}^{(t)}}_\infty^2
 \leq\operatorname{poly}(N)\rbr{1+\mathcal P_T^\Xi}.
\end{align*}

For the learning rate in \Cref{thm:entropy-main}, the preceding estimate gives
\begin{align*}
    80\eta\sum_{i=1}^N\sum_{t=1}^T\Xi_i
  \rbr{\lambda_i^{(t)}}^{3/2}
  \nbr{\rbr{\delta^N\bg_i}^{(t)}}_\infty^2
 -\frac1{2\eta}\sum_{i=1}^N\mathcal P_{i,T}\leq \operatorname{poly}(N).
\end{align*}

\section*{Acknowledgments}

Initial versions of the proofs were developed with the assistance of ChatGPT 5.5 and ChatGPT 5.6 Sol. The authors subsequently streamlined and revised the proofs substantially.

\bibliographystyle{plainnat}
\bibliography{main}

\newpage

\appendix

\section[Proof of the main theorem]{Proof of \Cref{thm:entropy-main}}
\label{sec:entropy-main-proof}

Set
\begin{align}
 \gamma\coloneqq\eta\Xi_{\max}.
 \label{eq:normalized-learning-rate}
\end{align}
For every integer $m\geq2$ and $\btheta\in\RR^m$, write
$\bJ_m^x(\btheta)\coloneqq\bJ_{\bx_m}(\btheta)$ whenever the response is
differentiable.  For every integer $m\geq2$ and linear map
$\bJ\colon\RR^m\to\RR^m$, write
\begin{align*}
 \nbr\bJ_{\infty\to1}
 \coloneqq\sup_{\nbr\bz_\infty\leq1}\nbr{\bJ\bz}_1.
\end{align*}

\begin{lemma}
\label{lem:square-root-entropy-geometry}
For every integer $m\geq2$ and $\btheta\in\RR^m$, the maximizer defining
$Q_m(\btheta)$ is unique and belongs to
$\operatorname{relint}\widetilde\Delta^m$.  For every integer $m\geq2$, the
map $Q_m$ is continuously differentiable and $\bJ_m^x$ is differentiable.
For every integer $m\geq2$, score $\btheta\in\RR^m$, and
direction $\bxi\in\RR^m$,
\begin{align}
 \nbr{\bJ_m^x(\btheta)}_{\infty\to1}
 \leq&4\Xi_m\sqrt{\lambda_m(\btheta)},
 \label{eq:filter-jacobian-size}\\
 \nbr{\left.\dfrac{\ud}{\ud r}
 \bJ_m^x(\btheta+r\bxi)\right|_{r=0}}_{\infty\to1}
 \leq&\Xi_m
 \sqrt{\inner{\bxi}{
 \left.\dfrac{\ud}{\ud r}\by_m(\btheta+r\bxi)\right|_{r=0}}}.
 \label{eq:filter-jacobian-variation}
\end{align}
For every integer $m\geq2$ and $\btheta,\be\in\RR^m$, set
$\btheta'\coloneqq\btheta+\be$,
$\bmu\coloneqq Q_m(\btheta)$, $\bmu'\coloneqq Q_m(\btheta')$, and
$\widehat\lambda\coloneqq\max\cbr{\lambda_m(\btheta),\lambda_m(\btheta')}$.  If
$\Xi_m\nbr\be_\infty\leq1/100$, then
\begin{align}
 \abr{\log\frac{\lambda_m(\btheta')}{\lambda_m(\btheta)}}
 \leq&\Xi_m\nbr\be_\infty,
 \label{eq:filter-mass-stability}\\
 \inner{\be}{\by_m(\btheta')-\by_m(\btheta)}
 \leq&20\sqrt2\Xi_m\widehat\lambda^{3/2}
       \nbr\be_\infty^2,
 \label{eq:lifted-local-smoothness}\\
 \sqrt{\widehat\lambda}
 \nbr{\bx_m(\btheta')-\bx_m(\btheta)}_1^2
 \leq&16\sqrt2\Xi_m
 \min\cbr{D_{\widetilde\psi_m}(\bmu',\bmu),
 D_{\widetilde\psi_m}(\bmu,\bmu')},
 \label{eq:lifted-local-norm-bound}\\
 D_{\widetilde\psi_m}^{\rm sym}(\bmu',\bmu)
 \leq&4
 \min\cbr{D_{\widetilde\psi_m}(\bmu',\bmu),
 D_{\widetilde\psi_m}(\bmu,\bmu')}.
 \label{eq:filter-directed-comparability}
\end{align}
\end{lemma}

The proof is deferred to \Cref{sec:power-local-response-geometry}.

For a filter
$\mathsf K=\sum_{s\geq0}k_s\mathsf L^s$, write
\begin{align*}
 \mathsf K(\varsigma)=\sum_{s\geq0}k_s\varsigma^s,
 \qquad
 \operatorname{ker}(\mathsf K)=(k_s)_{s\geq0}.
\end{align*}

\begin{lemma}
\label{lem:filter-kernel-bounds}
Set $\mathsf K_\star\coloneqq\mathsf A-\delta^{N+1}$, and, for every integer
$1\leq h\leq N+1$, set
$\mathsf K_h\coloneqq\mathsf A\delta^{h-1}$.  For every integer
$1\leq h\leq N+1$, the filters in \Cref{eq:stable-filter-law} satisfy
\begin{align}
 \delta=&\mathsf A\mathsf D,
 \qquad
 \mathsf K_h\mathsf D=\delta^h,
 \label{eq:filter-factorization}\\
 \mathsf K_h(I-\mathsf D\delta^N)=&
 \mathsf K_\star\delta^{h-1}.
 \label{eq:ema-controller-identity}
\end{align}
For every integer $1\leq h\leq N+1$,
$\mathsf K\in\cbr{\delta^h,\mathsf K_h,\mathsf K_\star}$, and
$0\leq\epsilon\leq1/(8N)$, if
$\operatorname{ker}(\mathsf K)=(k_s)_{s\geq0}$, then
\begin{align}
 \sum_{s\geq0}e^{\epsilon s/2}(s+1)^4\abr{k_s}
 \leq2^{28}N^6.
 \label{eq:filter-fourth-moment}
\end{align}
For every integer $1\leq h\leq N+1$ and
$\mathsf K\in\cbr{\delta^h,\mathsf K_h,\mathsf K_\star}$, the $\ell_1$
norm and the first four absolute moments of $\operatorname{ker}(\mathsf K)$
are at most $2^{28}N^6$.  For every integer $1\leq h\leq N+1$, the
coefficients of $\delta^h$ satisfy
$\sum_{s\geq0}\operatorname{ker}(\delta^h)_s=0$.

For every integer $1\leq h\leq N+1$ and
$\mathsf K\in\cbr{\delta^h,\mathsf K_h,\mathsf K_\star}$, consider in any
normed space two inputs to $\mathsf K$ that agree at positive times and are
equal, respectively, to a fixed vector $\boldsymbol\varphi$ and
$\boldsymbol 0$ at every nonpositive time.  If their output difference is
$\boldsymbol\tau$, then
\begin{align}
 \sum_{t=1}^{\infty}\nbr{\boldsymbol\tau^{(t)}}^2
 \leq2^{59}N^{13}\nbr{\boldsymbol\varphi}^2.
 \label{eq:filter-startup-tail}
\end{align}
\end{lemma}

\begin{lemma}
\label{lem:slow-weight-convolution}
Fix $T\geq1$ and $\epsilon\geq0$.  Let $\bz$ take values in a normed
space, with $\bz^{(t)}=\boldsymbol 0$ for $t\leq0$, and let
$\mathsf K=\sum_{s\geq0}k_s\mathsf L^s$ satisfy
$\sum_{s\geq0}e^{\epsilon s/2}\abr{k_s}<\infty$.
Suppose positive weights $w^{(t)}$ are defined for every integer $t\leq T$
and satisfy
\begin{align*}
 w^{(t)}\leq e^{\epsilon s}w^{(t-s)}
 \qquad(t\in[T],\ s\geq0).
\end{align*}
Then
\begin{align*}
 \sum_{t=1}^Tw^{(t)}\nbr{\rbr{\mathsf K\bz}^{(t)}}^2
 \leq\rbr{\sum_{s\geq0}e^{\epsilon s/2}\abr{k_s}}^2
 \sum_{t=1}^Tw^{(t)}\nbr{\bz^{(t)}}^2.
\end{align*}
\end{lemma}

\begin{lemma}
\label{lem:filter-score-locality}
For every player $i\in[N]$ and time $t\geq1$,
\begin{align}
 \nbr{\be_i^{(t)}}_\infty\leq3\cdot2^{28}N^6,
 \qquad
 \nbr{\rbr{(I-\mathsf D\delta^N)\bg_i}^{(t)}}_\infty
 \leq 3\cdot2^{28}N^6.
 \label{eq:filter-pointwise-score-bound}
\end{align}
If $3\cdot2^{28}\gamma N^6\leq1/100$, then, for every player $i\in[N]$
and time $t\geq1$, the two scores in \Cref{eq:lifted-oftrl-update} at times
$t$ and $t+1$ satisfy the hypothesis of
\Cref{lem:square-root-entropy-geometry}, and
\begin{align}
 \abr{\log\frac{\lambda_i^{(t+1)}}{\lambda_i^{(t)}}}
 \leq3\cdot2^{28}\gamma N^6.
 \label{eq:played-mass-drift}
\end{align}
\end{lemma}

For every player $i\in[N]$ and time $t\geq1$, define
$\btheta_i^{(t)}$, and for every player $i\in[N]$, define $\bd_i$ by
\begin{align*}
 \btheta_i^{(t)}
 \coloneqq\eta\rbr{\bS_i^{(t-1)}+\widehat\bg_i^{(t)}},
 \qquad
 \bd_i\coloneqq(I-\mathsf D\delta^N)\bg_i.
\end{align*}
For every player $i\in[N]$ and integer $t\leq0$, set
$\btheta_i^{(t)}=\boldsymbol 0$.  For every player $i\in[N]$, both
$\btheta_i^{(1)}-\btheta_i^{(0)}$ and $\bd_i^{(1)}$ vanish.  For every
$i\in[N]$ and integer $t\geq2$,
\begin{align*}
 \btheta_i^{(t)}-\btheta_i^{(t-1)}
 =\eta\rbr{\bg_i^{(t-1)}+\widehat\bg_i^{(t)}
 -\widehat\bg_i^{(t-1)}}
 =\eta\rbr{\widehat\bg_i^{(t)}+\be_i^{(t-1)}}
 =\eta\bd_i^{(t)}.
\end{align*}
For every player $i\in[N]$ and time $t\geq1$, define
\begin{align*}
 \overline{\bJ}_i^{(t)}
 \coloneqq\eta\int_0^1\bJ_{m_i}^x\rbr{
 \btheta_i^{(t-1)}+\alpha\eta\bd_i^{(t)}}\ud\alpha.
\end{align*}
For every $i\in[N]$ and integer $t\leq0$, set
$\overline{\bJ}_i^{(t)}=\eta\bJ_{m_i}^x(\boldsymbol 0)$.  The fundamental
theorem of calculus gives, for every $i\in[N]$ and integer $t$,
\begin{align}
 \rbr{\mathsf D\bx_i}^{(t)}
 =\overline{\bJ}_i^{(t)}\bd_i^{(t)}.
 \label{eq:direct-action-path}
\end{align}

\begin{lemma}
\label{lem:response-secant-estimates}
Suppose $3\cdot2^{28}\gamma N^6\leq1/100$.  For every $i\in[N]$ and
$t\geq1$,
\begin{align}
 \nbr{\overline{\bJ}_i^{(t)}}_{\infty\to1}
 \leq&4\sqrt2\eta\Xi_i\sqrt{\lambda_i^{(t)}}
 \leq4\sqrt2\gamma\sqrt{\lambda_i^{(t)}},
 \label{eq:scaled-secant-size}\\
 \sqrt{\lambda_i^{(t)}}
 \nbr{\rbr{\mathsf D\bx_i}^{(t)}}_1^2
 \leq&16\sqrt2\Xi_i
 D_{\widetilde\psi_i}\rbr{\bmu_i^{(t)},\bmu_i^{(t-1)}},
 \label{eq:movement-path-bound}\\
 \nbr{\rbr{\mathsf D\bx_i}^{(t)}}_1
 \leq&3\sqrt2\cdot2^{30}\eta\Xi_iN^6\sqrt{\lambda_i^{(t)}}
 \leq3\sqrt2\cdot2^{30}\gamma N^6\sqrt{\lambda_i^{(t)}}.
 \label{eq:pointwise-action-movement}
\end{align}
For every $i\in[N]$ and horizon $T\geq1$,
\begin{align}
 \sum_{t=1}^T
 \nbr{\rbr{\mathsf D\overline{\bJ}_i}^{(t)}}_{\infty\to1}^2
 \leq16\eta^2\Xi_i^2\mathcal P_{i,T}.
 \label{eq:consecutive-secant-variation}
\end{align}
\end{lemma}

The proofs of these four lemmas are deferred to
\Cref{sec:polynomial-closure}.

\subsection{The optimistic-FTRL path inequality}

Fix $T\geq1$, $i\in[N]$, and a comparator
$\widehat\bmu_i=(\widehat\lambda_i,\widehat\by_i)
\in\widetilde\Delta_i$.  Use $\mathcal P_{i,T}$ from
\Cref{sec:multilinear-analysis} and suppress the player index $i$ below.
Set $\be^{(0)}\coloneqq\boldsymbol 0$.  For $0\leq t\leq T$, write
\begin{align*}
 \bmu^{(t)}=(\lambda^{(t)},\by^{(t)})
 \coloneqq Q\rbr{\eta\rbr{\bS^{(t)}-\be^{(t)}}},
\end{align*}
and set
$\bmu^{(T+1)}=(\lambda^{(T+1)},\by^{(T+1)})
\coloneqq Q\rbr{\eta\bS^{(T)}}$.

The increments of the scores defining $\bmu^{(0)},\ldots,\bmu^{(T+1)}$,
divided by $\eta$, are $\bg^{(t)}-(\be^{(t)}-\be^{(t-1)})$ for $t\in[T]$
and $\be^{(T)}$ for $t=T+1$.  Reindexing the terms involving $\be$ and
applying the standard regularized-leader telescope give
\begin{align*}
 \sum_{t=1}^T\inner{\bg^{(t)}}{\widehat\by-\by^{(t)}}
 =&\sum_{t=1}^T
 \inner{\bg^{(t)}-\rbr{\be^{(t)}-\be^{(t-1)}}}
 {\widehat\by-\by^{(t)}}
 +\inner{\be^{(T)}}{\widehat\by-\by^{(T+1)}}\\
 &+\sum_{t=1}^T
 \inner{\be^{(t)}}{\by^{(t+1)}-\by^{(t)}}\\
 \leq&\frac{\widetilde\psi(\widehat\bmu)
 -\widetilde\psi(\bmu^{(0)})}{\eta}
 -\frac1\eta\sum_{t=1}^{T+1}
 D_{\widetilde\psi}\rbr{\bmu^{(t)},\bmu^{(t-1)}}
 +\sum_{t=1}^T
 \inner{\be^{(t)}}{\by^{(t+1)}-\by^{(t)}}.
\end{align*}

Suppose $2\eta\Xi_i\nbr{\be^{(t)}}_\infty\leq1/100$ for every $t\in[T]$.
For each $t\in[T]$, set
$\widetilde\bmu^{(t)}=(\widetilde\lambda^{(t)},\widetilde\by^{(t)})\coloneqq
Q\rbr{\eta\rbr{\bS^{(t)}+\be^{(t)}}}$.
For every $t\in[T]$, \Cref{eq:filter-mass-stability,eq:lifted-local-smoothness}
and $\exp\rbr{3/200}\leq\sqrt2$ give
\begin{align*}
 D_{\widetilde\psi}\rbr{\bmu^{(t)},\widetilde\bmu^{(t)}}
 \leq&D_{\widetilde\psi}^{\rm sym}
 \rbr{\widetilde\bmu^{(t)},\bmu^{(t)}}\\
 =&\inner{2\eta\be^{(t)}}
 {\widetilde\by^{(t)}-\by^{(t)}}\\
 \leq&20\sqrt2\Xi_i
 \rbr{\sqrt2\rbr{\lambda^{(t)}}^{3/2}}
 \rbr{2\eta}^2\nbr{\be^{(t)}}_\infty^2\\
 =&160\eta^2\Xi_i\rbr{\lambda^{(t)}}^{3/2}
 \nbr{\be^{(t)}}_\infty^2.
\end{align*}
For every $t\in[T]$, the Bregman three-point identity therefore gives
\begin{align*}
 2\eta\inner{\be^{(t)}}{\by^{(t+1)}-\by^{(t)}}
 =&D_{\widetilde\psi}\rbr{\bmu^{(t+1)},\bmu^{(t)}}
 +D_{\widetilde\psi}\rbr{\bmu^{(t)},\widetilde\bmu^{(t)}}
 -D_{\widetilde\psi}\rbr{\bmu^{(t+1)},\widetilde\bmu^{(t)}}\\*
 \leq&D_{\widetilde\psi}\rbr{\bmu^{(t+1)},\bmu^{(t)}}
 +160\eta^2\Xi_i\rbr{\lambda^{(t)}}^{3/2}
 \nbr{\be^{(t)}}_\infty^2.
\end{align*}
Summing over $t\in[T]$ and using
$\bmu^{(1)}=\bmu^{(0)}=Q\rbr{\boldsymbol 0}$,
\begin{align*}
 -\frac1\eta\sum_{t=1}^{T+1}
 D_{\widetilde\psi}\rbr{\bmu^{(t)},\bmu^{(t-1)}}
 +\frac1{2\eta}\sum_{t=1}^T
 D_{\widetilde\psi}\rbr{\bmu^{(t+1)},\bmu^{(t)}}
 =-\frac1{2\eta}\sum_{t=2}^{T+1}
 D_{\widetilde\psi}\rbr{\bmu^{(t)},\bmu^{(t-1)}}
 \leq-\frac{\mathcal P_{i,T}}{2\eta}.
\end{align*}
\begin{samepage}
Using the preceding bounds and the range estimate $4\Gamma_i$ implied by
\Cref{lem:convex-power-entropy}, we maximize over the comparator, sum over
players, and apply \Cref{eq:positive-part-lift} to obtain
\begin{align}
 \sum_{i=1}^N[\Reg_i(T)]_+
 \leq\frac4\eta\sum_{i=1}^N\Gamma_i
 +80\eta\sum_{i=1}^N\sum_{t=1}^T
 \Xi_i\rbr{\lambda_i^{(t)}}^{3/2}\nbr{\be_i^{(t)}}_\infty^2
 -\frac1{2\eta}\sum_{i=1}^N\mathcal P_{i,T}.
 \label{eq:lifted-rvu}
\end{align}
\end{samepage}

\subsection{Weighted variation bound}

\begin{proposition}
\label{prop:multilinear-variation-transfer}
If
\begin{align}
 \gamma\leq\frac{2^{-94}}{N^{20}},
 \label{eq:multilinear-small-rate}
\end{align}
then, for every $T\geq1$,
\begin{align}
 \sum_{i=1}^N\sum_{t=1}^T
 \rbr{\lambda_i^{(t)}}^{3/2}
 \nbr{\rbr{\delta^N\bg_i}^{(t)}}_\infty^2
 \leq2^{204}N^{16}\rbr{1+\mathcal P_T^\Xi}.
 \label{eq:multilinear-transfer}
\end{align}
\end{proposition}

\begin{proof}
Use the extensions $(\bg_i,\bx_i,\lambda_i)_{i\in[N]}$ from
\Cref{sec:finite-algorithm}, together with
$\overline{\bJ}_i^{(t)}=\eta\bJ_{m_i}^x(\boldsymbol 0)$ for every $i\in[N]$
and $t\leq0$.
For every integer $1\leq h\leq N$, \Cref{eq:ema-controller-identity} gives
\begin{align*}
 \mathsf K_h\rbr{I-\mathsf D\delta^N}
 =\mathsf K_\star\delta^{h-1}.
\end{align*}

Fix $r,j\in[N]$ and $S\subseteq[N]\setminus\cbr r$.  Since $|S|\leq N-1$,
\begin{align*}
 \frac32+|S|\leq N+\frac12.
\end{align*}
Hence \Cref{eq:played-mass-drift,eq:multilinear-small-rate} and $N\geq2$ give,
for every $t\in[T]$,
\begin{align*}
 \abr{\log\frac{w_{r,S}^{(t+1)}}{w_{r,S}^{(t)}}}
 \leq&3\cdot2^{28}\rbr{\frac32+|S|}\gamma N^6\\
 \leq&\frac92\,2^{28}\gamma N^7
 \leq\frac92\,2^{-66}N^{-13}
 \leq\frac1{8N}.
\end{align*}
The same rate condition also gives
\begin{align*}
 3\cdot2^{28}\gamma N^6
 \leq&3\cdot2^{-66}N^{-14}\leq\frac1{100},\\
 3\cdot2^{28}\gamma N^7
 \leq&3\cdot2^{-66}N^{-13}\leq\frac14,\\
 \gamma\leq&N^{-20}.
\end{align*}
These inequalities verify the numerical hypotheses of
\Cref{lem:filtered-action-path-bound,lem:filtered-response-identity,lem:filtered-multiaffine-expansion}.
For every $i\in[N]$,
$\lambda_i^{(1)}=\lambda_{m_i}(\boldsymbol 0)=\lambda_i^{(0)}$.
Telescoping the logarithmic bound and using
$\lambda_i^{(\tau)}=\lambda_i^{(0)}$ for every $i\in[N]$ and integer
$\tau\leq0$ gives, for every $t\in[T]$ and $s\geq0$,
\begin{align*}
 w_{r,S}^{(t)}\leq e^{s/(8N)}w_{r,S}^{(t-s)}.
\end{align*}
The kernel bound in \Cref{eq:filter-fourth-moment} therefore verifies the
remaining hypothesis of \Cref{lem:slow-weight-convolution} with $w=w_{r,S}$
and $\mathsf K\in\cbr{\mathsf K_\star,\mathsf K_h}$ for every integer
$1\leq h\leq N$.

\begin{lemma}
\label{lem:filtered-action-path-bound}
Suppose $3\cdot2^{28}\gamma N^6\leq1/100$ and
$3\cdot2^{28}\gamma N^7\leq1/4$.  For every $j\in[N]$, integer
$1\leq h\leq N$, and horizon $T\geq1$,
\begin{align*}
 \sum_{t=1}^T\sqrt{\lambda_j^{(t)}}
 \nbr{\rbr{\delta^h\bx_j}^{(t)}}_1^2
 \leq2^{60}\sqrt2\,N^{12}\Xi_j\mathcal P_{j,T}.
\end{align*}
\end{lemma}
The proof is deferred to \Cref{sec:response-secants-locality}.

There are two cases.  If $j\in S\cup\cbr r$, then, for every $t\in[T]$,
$w_{r,S}^{(t)}\leq\sqrt{\lambda_j^{(t)}}$, so
\Cref{lem:filtered-action-path-bound} gives
\begin{align}
 \cE_j^{N-|S|}(w_{r,S})
 \leq2^{61}N^{12}\mathcal P_T^\Xi.
 \label{eq:finite-label-direct}
\end{align}

For every player $i\in[N]$, regard $\overline{\bJ}_i$ as pointwise
multiplication on score sequences.  For every $i\in[N]$, integer
$1\leq h\leq N$, and score sequence $\bz$, define
\begin{align*}
 [\mathsf K_h,\overline{\bJ}_i]\bz
 \coloneqq\mathsf K_h(\overline{\bJ}_i\bz)
 -\overline{\bJ}_i(\mathsf K_h\bz).
\end{align*}
For every $i\in[N]$ and integer $1\leq h\leq N$, define
\begin{align*}
 \bc_{i,h}\coloneqq[\mathsf K_h,\overline{\bJ}_i]\bd_i.
\end{align*}

\begin{lemma}
\label{lem:filtered-response-identity}
Suppose $3\cdot2^{28}\gamma N^6\leq1/100$.  For every $i\in[N]$,
integer $1\leq h\leq N$, and time $t\geq1$,
\begin{align}
 \rbr{\delta^h\bx_i}^{(t)}
 =\overline{\bJ}_i^{(t)}
 \rbr{\mathsf K_\star\delta^{h-1}\bg_i}^{(t)}
 +\bc_{i,h}^{(t)}.
 \label{eq:filtered-response-order-reduction}
\end{align}
For every $i\in[N]$, integer $1\leq h\leq N$, and horizon $T\geq1$,
\begin{align}
 \sum_{t=1}^T\nbr{\bc_{i,h}^{(t)}}_1^2
 \leq9\cdot2^{116}\eta^2\Xi_i^2N^{24}\mathcal P_{i,T}
 \leq9\cdot2^{116}\gamma^2N^{24}\mathcal P_{i,T}.
 \label{eq:filtered-response-residual-bound}
\end{align}
\end{lemma}
The proof is deferred to \Cref{sec:response-secants-locality}.

Suppose instead that $j\notin S\cup\cbr r$, and set $h=N-|S|$.  Then
$|S|\leq N-2$, so $h\geq2$.
By \Cref{lem:filtered-response-identity,%
eq:filtered-response-order-reduction},
\begin{align}
 \cE_j^{N-|S|}(w_{r,S})
 \leq2\sum_{t=1}^T w_{r,S}^{(t)}
 \nbr{\overline{\bJ}_j^{(t)}
 \rbr{\mathsf K_\star\delta^{h-1}\bg_j}^{(t)}}_1^2
 +2\sum_{t=1}^T w_{r,S}^{(t)}\nbr{\bc_{j,h}^{(t)}}_1^2.
 \label{eq:finite-label-response-substitution}
\end{align}

\begin{lemma}
\label{lem:filtered-multiaffine-expansion}
Let
\begin{align*}
 C_{\mathrm{multi}}
 \coloneqq2\max\cbr{2^{59},
 9\sqrt2\cdot2^{121}}.
\end{align*}
Suppose $3\cdot2^{28}\gamma N^6\leq1/100$.
For every $i,j\in[N]$ and profile
$\bx\in\prod_{k=1}^N\Delta^{m_k}$,
\begin{align*}
 \sup_{\bxi\in\RR^{m_j},\ \nbr{\bxi}_1\leq1}
 \nbr{\partial_j\mathcal G_i(\bx)[\bxi]}_\infty
 \leq2.
\end{align*}
For every $i\in[N]$ and integer $1\leq h\leq N$, there is a sequence
$\boldsymbol\chi_{i,h}$ such that, for every $t\geq1$,
\begin{align}
 \rbr{\delta^h\bg_i}^{(t)}
 =\sum_{j=1}^N \partial_j\mathcal G_i\rbr{\bx^{(t)}}
 \rbr{\delta^h\bx_j}^{(t)}+\boldsymbol\chi_{i,h}^{(t)}.
 \label{eq:filtered-multiaffine-expansion}
\end{align}
For every $i\in[N]$, integer $1\leq h\leq N$, and horizon $T\geq1$,
\begin{align}
 \sum_{t=1}^T\nbr{\boldsymbol\chi_{i,h}^{(t)}}_\infty^2
 \leq C_{\mathrm{multi}}\rbr{
 N^{13}+\gamma^2N^{27}\mathcal P_T^\Xi}.
 \label{eq:filtered-multiaffine-residual-sharp}
\end{align}
If $\gamma\leq N^{-20}$, then, for every $i\in[N]$, integer
$1\leq h\leq N$, and horizon $T\geq1$,
\begin{align}
 \sum_{t=1}^T\nbr{\boldsymbol\chi_{i,h}^{(t)}}_\infty^2
 \leq C_{\mathrm{multi}}N^{13}
 \rbr{1+\mathcal P_T^\Xi}.
 \label{eq:filtered-multiaffine-residual}
\end{align}
\end{lemma}
The proof is deferred to \Cref{sec:filtered-multiaffine-proof}.

Since
$C_{\mathrm{multi}}=2\max\cbr{2^{59},9\sqrt2\cdot2^{121}}$,
$9\sqrt2<16$, and $9<16$,
\begin{align*}
 C_{\mathrm{multi}}<2^{128},
 \qquad
 9\cdot2^{116}<2^{120}.
\end{align*}
By \Cref{lem:response-secant-estimates,eq:scaled-secant-size},
$\nbr{\overline{\bJ}_j^{(t)}}_{\infty\to1}
\leq8\gamma\sqrt{\lambda_j^{(t)}}$ for every $t\geq1$.  Since
$j\notin S\cup\cbr r$, for every integer $t\leq T$,
\begin{align*}
 w_{r,S}^{(t)}\lambda_j^{(t)}
 =&w_{r,S\cup\cbr j}^{(t)},\\
 h-1=&N-|S|-1=N-\abr{S\cup\cbr j}.
\end{align*}
The weight bound above, with $S\cup\cbr j$ in place of $S$, verifies the
hypotheses of \Cref{lem:slow-weight-convolution} for
$w=w_{r,S\cup\cbr j}$.  Applying that lemma with
$\mathsf K=\mathsf K_\star$, followed by
\Cref{eq:filtered-multiaffine-expansion} at order $h-1$, gives
\begin{align*}
 &\sum_{t=1}^T w_{r,S}^{(t)}
 \nbr{\overline{\bJ}_j^{(t)}
 \rbr{\mathsf K_\star\delta^{h-1}\bg_j}^{(t)}}_1^2\\
 \leq&64\cdot2^{56}\gamma^2N^{12}
 \sum_{t=1}^T w_{r,S}^{(t)}\lambda_j^{(t)}
 \nbr{\rbr{\delta^{h-1}\bg_j}^{(t)}}_\infty^2\\
 \leq&512\cdot2^{56}\gamma^2N^{13}
 \sum_{k=1}^N\cE_k^{N-\abr{S\cup\cbr j}}
 \rbr{w_{r,S\cup\cbr j}}
 +128\cdot2^{56}C_{\mathrm{multi}}
 \gamma^2N^{25}\rbr{1+\mathcal P_T^\Xi}.
\end{align*}
Here \Cref{eq:filtered-multiaffine-residual} applies because
$w_{r,S}^{(t)}\lambda_j^{(t)}\leq1$ for every $t\in[T]$.  Moreover,
\Cref{eq:filtered-response-residual-bound} gives
\begin{align*}
 \sum_{t=1}^T w_{r,S}^{(t)}\nbr{\bc_{j,h}^{(t)}}_1^2
 \leq9\cdot2^{116}\gamma^2N^{24}\mathcal P_{j,T}.
\end{align*}
Substituting the last two displays into
\Cref{eq:finite-label-response-substitution} and using
$\gamma\leq N^{-20}$ gives
\begin{align}
 \cE_j^{N-|S|}(w_{r,S})
 \leq2^{200}N^{13}\gamma^2
 \sum_{k=1}^N\cE_k^{N-\abr{S\cup\cbr j}}
 \rbr{w_{r,S\cup\cbr j}}
 +2^{200}N^{12}\rbr{1+\mathcal P_T^\Xi}.
 \label{eq:finite-label-transfer}
\end{align}

Since $r$, $j$, and $S$ were arbitrary, the preceding two case bounds hold
throughout their stated ranges.  For every integer $0\leq s\leq N-1$, define
\begin{align}
 M_s\coloneqq\max\cbr{\cE_i^{N-s}(w_{r,S}):
 r,i\in[N],\ S\subseteq[N]\setminus\cbr r,\ |S|=s}.
 \label{eq:finite-label-maximum}
\end{align}
When $s=N-1$, every feasible pair $(r,S)$ in
\Cref{eq:finite-label-maximum} satisfies
$S=[N]\setminus\cbr r$.  Hence $j\in S\cup\cbr r$ for every $j\in[N]$, and
\Cref{eq:finite-label-direct} applies at order one.  For every integer
$0\leq s<N-1$,
\Cref{eq:finite-label-direct,eq:finite-label-transfer} give
\begin{align*}
 M_s
 \leq\max\cbr{
 2^{200}N^{12}\mathcal P_T^\Xi,
 2^{200}N^{14}\gamma^2M_{s+1}
 +2^{200}N^{12}\rbr{1+\mathcal P_T^\Xi}}.
\end{align*}
Moreover,
\begin{align*}
 2^{200}N^{14}\gamma^2
 \leq2^{12}N^{-26}\leq2^{-14}\leq\frac12.
\end{align*}
Backward induction thus proves, for every integer $0\leq s\leq N-1$,
\begin{align}
 M_s\leq2^{201}N^{12}\rbr{1+\mathcal P_T^\Xi}.
 \label{eq:finite-label-induction}
\end{align}

Finally, apply \Cref{eq:filtered-multiaffine-expansion} at order $N$ and use
\Cref{eq:finite-label-induction} with $s=0$.  For every $r\in[N]$,
\begin{align*}
 \sum_{t=1}^T\rbr{\lambda_r^{(t)}}^{3/2}
 \nbr{\rbr{\delta^N\bg_r}^{(t)}}_\infty^2
 \leq8N\sum_{j=1}^N\cE_j^N(w_{r,\emptyset})
 +2C_{\mathrm{multi}}N^{13}\rbr{1+\mathcal P_T^\Xi}.
\end{align*}
Summing over $r$, using $C_{\mathrm{multi}}<2^{128}$, and applying
\Cref{eq:finite-label-induction} prove
\begin{align*}
 \sum_{r=1}^N\sum_{t=1}^T
 \rbr{\lambda_r^{(t)}}^{3/2}
 \nbr{\rbr{\delta^N\bg_r}^{(t)}}_\infty^2
 \leq&\rbr{2^{204}N^{15}+2^{129}N^{14}}
 \rbr{1+\mathcal P_T^\Xi}\\
 \leq&2^{204}N^{16}\rbr{1+\mathcal P_T^\Xi},
\end{align*}
as claimed.
\end{proof}

\begin{proof}[Proof of \Cref{thm:entropy-main}]
Fix $T\geq1$ and set $\eta=2^{-94}/(N^{20}\Xi_{\max})$.  Then
\Cref{eq:multilinear-small-rate} holds, and
\begin{align*}
 6\cdot2^{28}\eta\Xi_{\max}N^6
 =6\cdot2^{-66}N^{-14}\leq\frac1{100}.
\end{align*}
Thus, by \Cref{lem:filter-score-locality},
$2\eta\Xi_i\nbr{\be_i^{(t)}}_\infty\leq1/100$ for every $i\in[N]$ and
$t\in[T]$, as required in the derivation of \Cref{eq:lifted-rvu}.  By
\Cref{eq:stable-filter-prediction,eq:multilinear-transfer},
\begin{align*}
 80\eta\sum_{i=1}^N\sum_{t=1}^T
 \Xi_i\rbr{\lambda_i^{(t)}}^{3/2}\nbr{\be_i^{(t)}}_\infty^2
 \leq&80\cdot2^{110}N^{-4}
 \rbr{1+\mathcal P_T^\Xi},\\
 \frac1{2\eta}\sum_{i=1}^N\mathcal P_{i,T}
 \geq&2^{93}N^{20}\mathcal P_T^\Xi.
\end{align*}
Since $80\cdot2^{17}\leq2^{24}\leq N^{24}$,
the coefficient of $\mathcal P_T^\Xi$ after substitution into
\Cref{eq:lifted-rvu} is nonpositive.  Hence
\begin{align*}
 \sum_{i=1}^N[\Reg_i(T)]_+
 \leq4\cdot2^{94}N^{20}\Xi_{\max}\sum_{i=1}^N\Gamma_i
 +80\cdot2^{110}N^{-4}.
\end{align*}

The definitions of $\Gamma_i$, $\Xi_i$, and $m_{\max}$ give
\begin{align*}
 \Xi_{\max}\sum_{i=1}^N\Gamma_i
 \leq4\cdot10^4N
 \rbr{1+\log\rbr{m_{\max}+1}}^4.
\end{align*}
Consequently,
\begin{align*}
 \sum_{i=1}^N[\Reg_i(T)]_+
 \leq&5\cdot2^{109}N^{21}
 \rbr{1+\log\rbr{m_{\max}+1}}^4+80\cdot2^{110}N^{-4}\\
 \leq&7\cdot2^{109}N^{21}
 \rbr{1+\log\rbr{m_{\max}+1}}^4\\
 \leq&2^{112}N^{21}
 \rbr{1+\log\rbr{m_{\max}+1}}^4.
\end{align*}
\end{proof}

\section[Geometry of Qm]{Geometry of $Q_m$}
\label{sec:power-entropy-lift}

Fix an integer $m\geq2$.  Recall $\psi_m$ and $\Gamma_m$ from
\Cref{sec:finite-algorithm}:
\begin{align*}
 \psi_m(\bx)
 =\sum_{a=1}^m x(a)\log x(a),
 \qquad
 \Gamma_m
 = (\log m)^2+2\log m+2.
\end{align*}
Throughout this section, every dummy action index ranges over $[m]$; in
particular, $\sum_a\coloneqq\sum_{a\in[m]}$.

\subsection[Convexity of the lifted regularizer]{Convexity of $\widetilde\psi_m$}

For an interior $\bx\in\Delta^m$, define
\begin{align*}
 W_m(\bx)
 \coloneqq\sum_a x(a)
 \rbr{\log x(a)-\psi_m(\bx)}^2.
\end{align*}
Let $a\sim\bx$ and set
$\varepsilon\coloneqq-\log x(a)$.  For every $t\geq0$,
\begin{align*}
 \Pr(\varepsilon\geq t)
 =\sum_{a:x(a)\leq e^{-t}}x(a)
 \leq \min\cbr{1,me^{-t}},
\end{align*}
so the tail-integral formula gives
\begin{align*}
 \EE\sbr{\varepsilon^2}
 \leq&(\log m)^2+2\log m+2=\Gamma_m,\\
 \EE\sbr{\varepsilon^4}
 \leq&(\log m)^4+4(\log m)^3+12(\log m)^2
       +24\log m+24
 \leq 4\Gamma_m^2.
\end{align*}
The last inequality follows by expansion and holds for $m\geq2$.
Thus
\begin{align}
 W_m(\bx)=&\operatorname{Var}(\varepsilon)
 \leq\Gamma_m,
 \label{eq:entropy-varentropy-bound}\\
 \sum_a x(a)\abr{\log x(a)-\psi_m(\bx)}^4
 =&\EE\sbr{\abr{\varepsilon-\EE\sbr{\varepsilon}}^4}
 \leq8\rbr{\EE\sbr{\varepsilon^4}
       +\rbr{\EE\sbr{\varepsilon}}^4}
 \leq 64\Gamma_m^2.
 \label{eq:entropy-fourth-by-second}
\end{align}
For every $\bz\in\RR^m$ satisfying $\one^\top\bz=0$,
\begin{align}
 \inner{\nabla\psi_m(\bx)}{\bz}
 =&\sum_a\rbr{\log x(a)-\psi_m(\bx)}z(a).
 \nonumber\\
 \inner{\nabla\psi_m(\bx)}{\bz}^2
 \leq&\Gamma_m\sum_a\frac{z(a)^2}{x(a)}.
 \label{eq:entropy-intrinsic-gradient}
\end{align}
The last inequality follows from weighted Cauchy--Schwarz and
\Cref{eq:entropy-varentropy-bound}.

For $\lambda>0$ and $\by=\lambda\bx$, recall $\widetilde\psi_m$ from
\Cref{sec:finite-algorithm}:
\begin{align*}
 \widetilde\psi_m(\lambda,\by)
 =-
 \sqrt{1-\lambda}
 +\sqrt\lambda\sbr{\psi_m(\by/\lambda)-2\Gamma_m-3}
\end{align*}
on $\widetilde\Delta^m$, with
$\widetilde\psi_m(0,\boldsymbol0)=-1$.

\restate{lem:convex-power-entropy}

\begin{proof}
Fix an interior point $\bmu=(\lambda,\lambda\bx)$ and write a tangent vector
as $(s,s\bx+\lambda\bz)$, where $s\in\RR$ and $\one^\top\bz=0$.  Define
\begin{align*}
 \bv\coloneqq&
 \rbr{x(a)\rbr{\log x(a)-\psi_m(\bx)}}_{a\in[m]},\\
 \beta\coloneqq&2\Gamma_m+3-\psi_m(\bx)-W_m(\bx),
 \qquad
 \omega_\lambda\coloneqq
 \rbr{\frac{\lambda}{1-\lambda}}^{3/2}.
\end{align*}
Since $\one^\top\bz=0$,
\begin{align*}
 \inner{\nabla\psi_m(\bx)}{\bz}
 =\sum_a\frac{v(a)z(a)}{x(a)},
 \qquad
 W_m(\bx)
 = \sum_a\frac{v(a)^2}{x(a)}.
\end{align*}
Direct differentiation gives
\begin{align}
 &\rbr{s,s\bx+\lambda\bz}^{\!\top}
 \nabla^2\widetilde\psi_m(\bmu)
 \rbr{s,s\bx+\lambda\bz} \nonumber\\
 =&\sqrt\lambda\sum_a\frac{z(a)^2}{x(a)}
 -\lambda^{-1/2}s\sum_a\frac{v(a)z(a)}{x(a)}
 +\frac{2\Gamma_m+3-\psi_m(\bx)+\omega_\lambda}
 {4\lambda^{3/2}}s^2 \nonumber\\
 =&\sqrt\lambda\sum_a
 \frac{[z(a)-(s/(2\lambda))v(a)]^2}{x(a)}
 +\frac{2\Gamma_m+3-\psi_m(\bx)-W_m(\bx)+\omega_\lambda}
 {4\lambda^{3/2}}s^2.
 \label{eq:power-completed-square}
\end{align}
By \Cref{eq:entropy-varentropy-bound} and the bounds
$\psi_m(\bx)\leq0$ and
$-\psi_m(\bx)\leq\log m\leq\sqrt{\Gamma_m}$,
\begin{align}
 \Gamma_m+3
 \leq\beta
 \leq 3(1+\Gamma_m).
 \label{eq:power-curvature-lower}
\end{align}

By \Cref{eq:entropy-intrinsic-gradient} and Young's inequality,
\begin{align*}
 \lambda^{-1/2}\abr s
 \abr{\sum_a\frac{v(a)z(a)}{x(a)}}
 \leq
 \frac12\sqrt\lambda\sum_a\frac{z(a)^2}{x(a)}
 +\frac{\Gamma_m}{2\lambda^{3/2}}s^2.
\end{align*}
Substitution gives
\begin{align}
 \rbr{s,s\bx+\lambda\bz}^{\!\top}
 \nabla^2\widetilde\psi_m(\bmu)
 \rbr{s,s\bx+\lambda\bz}
 \geq\frac12\sqrt\lambda\sum_a\frac{z(a)^2}{x(a)}
 +\frac{3s^2}{4\lambda^{3/2}}
 +\frac{s^2}{4(1-\lambda)^{3/2}}.
 \label{eq:power-coercivity}
\end{align}
Continuity at $\lambda=0$,
convexity on the closed domain, and the range bound follow from
$-\log m\leq\psi_m(\bx)\leq0$.  The one-sided $\lambda$-derivatives of
$\widetilde\psi_m$ diverge at $\lambda\in\cbr{0,1}$, while
$\partial_a\psi_m(\bx)=1+\log x(a)\to-\infty$ as $x(a)\downarrow0$.  Hence
$Q_m(\btheta)$ is interior for every
$\btheta\in\RR^m$;
strict convexity gives uniqueness.
Fix $\btheta\in\RR^m$ and $\lambda'\in(0,1)$.  Entropy conjugacy gives
\begin{align*}
 \max_{\bx\in\Delta^m}\cbr{
 \lambda'\inner\btheta\bx-\sqrt{\lambda'}\,\psi_m(\bx)}
 =\sqrt{\lambda'}\log\rbr{\sum_a
 \exp\rbr{\sqrt{\lambda'}\,\theta(a)}},
\end{align*}
and its unique maximizer is
\begin{align*}
 \frac{\rbr{\exp\rbr{\sqrt{\lambda'}\,\theta(a)}}_{a\in[m]}}
 {\sum_a\exp\rbr{\sqrt{\lambda'}\,\theta(a)}}.
\end{align*}
Thus, after optimizing over $\bx$, the objective defining $Q_m(\btheta)$ as a
function of $\lambda'$ is
\begin{align*}
 \sqrt{1-\lambda'}
 +\sqrt{\lambda'}\sbr{2\Gamma_m+3
 +\log\rbr{\sum_a
 \exp\rbr{\sqrt{\lambda'}\,\theta(a)}}}.
\end{align*}
Its derivative is
\begin{align*}
 -\frac{1}{2\sqrt{1-\lambda'}}
 +\frac{2\Gamma_m+3
 +\log\rbr{\sum_a
 \exp\rbr{\sqrt{\lambda'}\,\theta(a)}}}{2\sqrt{\lambda'}}
 +\frac12
 \frac{\sum_a\theta(a)\exp\rbr{\sqrt{\lambda'}\,\theta(a)}}
 {\sum_a\exp\rbr{\sqrt{\lambda'}\,\theta(a)}}.
\end{align*}
Fix distinct $\lambda_0,\lambda_1\in(0,1)$.  For every
$j\in\cbr{0,1}$, let $\by_j\in\lambda_j\Delta^m$ maximize
$\inner\btheta\by-\widetilde\psi_m(\lambda_j,\by)$.  For every $t\in(0,1)$,
\begin{align*}
 (1-t)\by_0+t\by_1
 \in\rbr{(1-t)\lambda_0+t\lambda_1}\Delta^m.
\end{align*}
Strict concavity of the lifted objective therefore implies strict concavity
of the scalar objective, so its derivative is strictly decreasing.  The
function in the lemma is $-2$ times this derivative and is therefore strictly
increasing.
The derivative tends to $+\infty$ as $\lambda'\downarrow0$ and to $-\infty$
as $\lambda'\uparrow1$, which gives the two limits in the lemma.  Its unique
zero and the conditional maximizer give the formula for $Q_m(\btheta)$.
\end{proof}

\subsection[Bounds for Qm]{Bounds for $Q_m$}
\label{sec:power-local-response-geometry}

\restate{lem:square-root-entropy-geometry}
\begin{proof}
Fix an integer $m\geq2$ and $\btheta\in\RR^m$, and write
\begin{align*}
 Q_m(\btheta)=(\lambda,\lambda\bx).
\end{align*}
By \Cref{lem:convex-power-entropy}, the maximizer defining $Q_m(\btheta)$ is
unique and belongs to $\operatorname{relint}\widetilde\Delta^m$.
The Hessian in \Cref{eq:power-completed-square} is positive definite, so
the implicit-function theorem makes $Q_m$ and $\bx_m$ smooth.

\begin{align*}
 \bB\coloneqq&\diag(\bx)-\bx\bx^\top,
 \qquad
 \boldsymbol\iota\coloneqq
 \rbr{\log x(a)-\psi_m(\bx)}_{a\in[m]},\\
 \bv\coloneqq&\rbr{x(a)\iota(a)}_{a\in[m]},
 \qquad
 \bw\coloneqq \bx+\frac12\bv,\\
 \beta\coloneqq&2\Gamma_m+3-\psi_m(\bx)-W_m(\bx),
 \qquad
 \omega_\lambda\coloneqq
 \rbr{\frac\lambda{1-\lambda}}^{3/2}.
\end{align*}
For $\be\in\RR^m$, write
$\left.\frac{\ud}{\ud r}Q_m(\btheta+r\be)\right|_{r=0}
=(s,s\bx+\lambda\bz)$.  The linearized first-order
condition gives, for every
$\bxi\in\RR^m$ satisfying $\one^\top\bxi=0$,
\begin{align*}
 \sqrt\lambda\sum_a
 \frac{[z(a)-(s/(2\lambda))v(a)]\xi(a)}{x(a)}
 =&\lambda\inner{\be}{\bxi},
 \\
 \frac{\beta+\omega_\lambda}{4\lambda^{3/2}}s
 =&\inner{\be}{\bw}.
\end{align*}
Therefore,
\begin{align*}
 s
 =\frac{4\lambda^{3/2}}{\beta+\omega_\lambda}
 \inner{\be}{\bw},
 \qquad
 \bz
 = \sqrt\lambda\bB\be+\frac{s}{2\lambda}\bv.
\end{align*}
Consequently, $\bJ_m^x(\btheta)$ and $\bJ_{\by_m}(\btheta)$ are
\begin{align}
 \bJ_m^x(\btheta)
 =&\sqrt\lambda\sbr{
 \bB+\frac2{\beta+\omega_\lambda}\bv\bw^\top},
 \label{eq:power-normalized-jacobian}\\
 \bJ_{\by_m}(\btheta)
 =&\lambda^{3/2}\sbr{
 \bB+\frac4{\beta+\omega_\lambda}\bw\bw^\top}.
 \nonumber
\end{align}
By \Cref{eq:entropy-varentropy-bound,eq:power-curvature-lower},
\begin{align*}
 \nbr\bv_1
 \leq\sqrt{\Gamma_m},
 \qquad
 \nbr\bw_1
 \leq 1+\frac12\sqrt{\Gamma_m},
 \qquad
 \nbr\bB_{\infty\to1}
 \leq 2.
\end{align*}
Consequently,
\begin{align}
 \nbr{\bJ_m^x(\btheta)}_{\infty\to1}
 \leq&5\sqrt\lambda,
 \label{eq:power-jacobian-size}\\
 \be^\top\bJ_{\by_m}(\btheta)\be
 \leq&10\lambda^{3/2}\nbr\be_\infty^2,
 \label{eq:power-dual-score}\\
 \frac{\abr s}{\lambda}
 \leq&6\nbr\be_\infty.
 \label{eq:power-radial-comparability}
\end{align}
\Cref{eq:power-jacobian-size} proves \Cref{eq:filter-jacobian-size};
integrating \Cref{eq:power-radial-comparability} along
$\btheta+t\be$, $t\in[0,1]$, proves \Cref{eq:filter-mass-stability}.

Let $\boldsymbol\vartheta\colon[0,1]\to\RR^m$ be differentiable and set
\begin{align*}
 \bmu(r)
 \coloneqq Q_m(\boldsymbol\vartheta(r))
 = (\lambda(r),\lambda(r)\bx(r)).
\end{align*}
Fix $r\in[0,1]$, suppress the argument $r$, and set
\begin{align*}
 s\coloneqq&\dot\lambda,
 \qquad
 \tau\coloneqq s/\lambda,
 \qquad
 \mathcal E\coloneqq \sum_a\frac{\dot x(a)^2}{x(a)},\\
 \mathcal H\coloneqq&
 \dot\bmu^\top\nabla^2\widetilde\psi_m(\bmu)\dot\bmu.
\end{align*}
Differentiating $\bB$, $\psi_m$, $\bv$, $W_m$, and $\beta$ with respect to
$r$ gives
\begin{align*}
 \dot\bB
 =&\diag(\dot\bx)-\dot\bx\bx^\top-\bx\dot\bx^\top,\\
 \dot\psi_m
 =&\sum_a\iota(a)\dot x(a),\\
 \dot v(a)
 =&\dot x(a)(1+\iota(a))-x(a)\dot\psi_m,
 \qquad a\in[m],\\
 \dot W_m
 =&\sum_a\dot x(a)(\iota(a)^2+2\iota(a)),
 \qquad
 \dot\beta=-\dot\psi_m-\dot W_m.
\end{align*}
Weighted Cauchy--Schwarz and
\Cref{eq:entropy-varentropy-bound,eq:entropy-fourth-by-second} yield
\begin{align}
 \nbr{\dot\bB}_{\infty\to1}
 \leq3\sqrt{\mathcal E},
 \qquad
 \nbr{\dot\bv}_1+\nbr{\dot\bw}_1
 \leq 6\sqrt{1+\Gamma_m}\sqrt{\mathcal E},
 \qquad
 \abr{\dot\beta}
 \leq 13(1+\Gamma_m)\sqrt{\mathcal E}.
 \label{eq:power-block-variation}
\end{align}
The two scalar estimates
\begin{align}
 \frac{\sqrt\lambda\,\omega_\lambda
 (\lambda^{-1}+(1-\lambda)^{-1})}
 {(\beta+\omega_\lambda)^2}
 \leq&3\sqrt{\lambda^{-3/2}+(1-\lambda)^{-3/2}},
 \label{eq:power-radial-scalar-one}\\
 \frac{\lambda^{3/2}\omega_\lambda
 (\lambda^{-1}+(1-\lambda)^{-1})}
 {(\beta+\omega_\lambda)^2}
 \leq&\sqrt2
 \label{eq:power-radial-scalar-two}
\end{align}
follow by splitting at $\lambda=1/2$.

The derivative of \Cref{eq:power-normalized-jacobian} is
\begin{align*}
 \dot\bJ_m^x
 =&\frac12\tau\bJ_m^x
 +\sqrt\lambda\sbr{
 \dot\bB
 +\frac{2(\dot\bv\bw^\top+\bv\dot\bw^\top)}
 {\beta+\omega_\lambda}
 -\frac{2\bv\bw^\top(\dot\beta+\dot\omega_\lambda)}
 {(\beta+\omega_\lambda)^2}},\\
 \dot\omega_\lambda
 =&\frac32\omega_\lambda s
 \rbr{\lambda^{-1}+(1-\lambda)^{-1}}.
\end{align*}
Substituting \Cref{eq:power-block-variation,eq:power-jacobian-size,%
eq:power-radial-scalar-one} and the bounds on $\bv,\bw,\beta$ gives
\begin{align*}
 \nbr{\dot\bJ_m^x}_{\infty\to1}
 \leq200(1+\Gamma_m)\sbr{
 \sqrt\lambda\sqrt{\mathcal E}
 +\sqrt\lambda\abr\tau
 +\abr s\sqrt{\lambda^{-3/2}+(1-\lambda)^{-3/2}}}.
\end{align*}
By \Cref{eq:power-coercivity},
\begin{align*}
 \sqrt\lambda\sqrt{\mathcal E}
 \leq\sqrt2\sqrt{\mathcal H},
 \qquad
 \sqrt\lambda\abr\tau
 \leq \frac2{\sqrt3}\sqrt{\mathcal H},
 \qquad
 \abr s\sqrt{\lambda^{-3/2}+(1-\lambda)^{-3/2}}
 \leq 2\sqrt{\mathcal H}.
\end{align*}
Hence
\begin{align}
 \nbr{\dot\bJ_m^x}_{\infty\to1}
 \leq10^3(1+\Gamma_m)\sqrt{\mathcal H}
 \leq \Xi_m\sqrt{\mathcal H}.
 \label{eq:power-metric-variation}
\end{align}
For an arbitrary $\be\in\RR^m$, specialize the preceding path to
$\boldsymbol\vartheta(r)=\btheta+r\be$, $r\in[0,1]$.  At $r=0$, the
linearized response optimality condition gives
\begin{align*}
 \mathcal H
 =\inner{\be}{
 \left.\dfrac{\ud}{\ud r}\by_m(\btheta+r\be)\right|_{r=0}}.
\end{align*}
Applying \Cref{eq:power-metric-variation} at $r=0$ proves
\Cref{eq:filter-jacobian-variation}.

Fix $\be\in\RR^m$.  Along $Q_m(\btheta+t\be)$, $t\in[0,1]$, let overdots
denote derivatives with respect to $t$, set $s\coloneqq\dot\lambda$ and
$\tau\coloneqq s/\lambda$, and use $\lambda$, $\bx$, $\bB$,
$\boldsymbol\iota$, $\bv$, $\bw$, $\beta$, and $\omega_\lambda$ for their
values at $t$.  Define
\begin{align*}
 \be^\circ\coloneqq&\be-\inner{\be}{\bx}\one,
 \qquad
 \varpi\coloneqq \inner{\be}{\bw},\\
 \kappa\coloneqq&\frac{4\lambda^{3/2}}
 {\beta+\omega_\lambda},
 \qquad
 \mathcal T\coloneqq \lambda^{3/2}\sum_ax(a)e^\circ(a)^2,\\
 \mathcal S\coloneqq&\kappa\varpi^2.
\end{align*}
Applying the formulas for $s$ and $\bz$ above with $\bz=\dot\bx$ at
$Q_m(\btheta+t\be)$ gives, for every $a\in[m]$,
\begin{align*}
 \frac{\dot x(a)}{x(a)}
 =\sqrt\lambda e^\circ(a)+\frac12\tau\iota(a),
 \qquad
 s= \kappa\varpi,
 \qquad
 \be^\top\bJ_{\by_m}(\btheta+t\be)\be
 = \mathcal T+\mathcal S.
\end{align*}
Differentiating $\mathcal T$, $\beta$, $\varpi$, $\kappa$, and $\mathcal S$ with
respect to $t$ gives
\begin{align}
 \dot{\mathcal T}
 =&\frac32\tau\mathcal T
 +\lambda^2\sum_ax(a)e^\circ(a)^3
 +\frac12\tau\lambda^{3/2}\sum_ax(a)e^\circ(a)^2\iota(a),
 \label{eq:power-tangent-variation}\\
 \dot\beta
 =&-\sum_ax(a)(\iota(a)^2+3\iota(a))
 \rbr{\sqrt\lambda e^\circ(a)+\frac12\tau\iota(a)},
 \nonumber\\
 \dot\varpi
 =&\frac32\sum_ax(a)e^\circ(a)
 \rbr{\sqrt\lambda e^\circ(a)+\frac12\tau\iota(a)}
 +\frac12\sum_ax(a)e^\circ(a)\iota(a)
 \rbr{\sqrt\lambda e^\circ(a)+\frac12\tau\iota(a)},
 \nonumber\\
 \frac{\dot\kappa}{\kappa}
 =&\frac32\tau-
 \frac{\dot\beta+\dot\omega_\lambda}{\beta+\omega_\lambda},
 \qquad
 \dot{\mathcal S}
 = \frac{\dot\kappa}{\kappa}\mathcal S+2\kappa\varpi\dot\varpi.
 \label{eq:power-radial-energy-variation}
\end{align}
Using $\abr{e^\circ(a)}\leq2\nbr\be_\infty$ for every $a\in[m]$,
\Cref{eq:entropy-varentropy-bound,eq:entropy-fourth-by-second,%
eq:power-radial-scalar-two}, and weighted Cauchy--Schwarz gives
\begin{align*}
 \abr\tau
 \leq&6\nbr\be_\infty,
 \qquad
 \frac{\abr{\dot\kappa}}{\kappa}
 \leq130\sqrt{1+\Gamma_m}\nbr\be_\infty,\\
 \sqrt\kappa\abr{\dot\varpi}
 \leq&8\nbr\be_\infty\rbr{\sqrt{\mathcal T}+\sqrt{\mathcal S}}.
\end{align*}
Substitution in \Cref{eq:power-tangent-variation,%
eq:power-radial-energy-variation} gives
\begin{align*}
 \abr{\dot{\mathcal T}}
 \leq&12\nbr\be_\infty(\mathcal T+\mathcal S),\\
 \abr{\dot{\mathcal S}}
 \leq&160(1+\Gamma_m)\nbr\be_\infty(\mathcal T+\mathcal S).
\end{align*}
Therefore, whenever $\mathcal T+\mathcal S>0$,
\begin{align}
 \abr{\frac{\ud}{\ud t}\log(\mathcal T+\mathcal S)}
 \leq200(1+\Gamma_m)\nbr\be_\infty
 \leq \Xi_m\nbr\be_\infty.
 \label{eq:power-aligned-stability}
\end{align}

Suppose $\Xi_m\nbr\be_\infty\leq1/100$, and set
\begin{align*}
 \bmu_t\coloneqq Q_m(\btheta+t\be),
 \qquad t\in[0,1],
 \qquad
 \mathcal H(t)\coloneqq
 \be^\top\bJ_{\by_m}(\btheta+t\be)\be.
\end{align*}
If $\be\neq\boldsymbol0$, positive definiteness and
\Cref{eq:power-aligned-stability} imply
$\max_{t\in[0,1]}\mathcal H(t)\leq
2\min_{t\in[0,1]}\mathcal H(t)$.  Fenchel duality and
Taylor's integral formula give
\begin{align*}
 D_{\widetilde\psi_m}(\bmu_1,\bmu_0)
 =&\int_0^1t\mathcal H(t)\,\ud t,
 \qquad
 D_{\widetilde\psi_m}(\bmu_0,\bmu_1)
 = \int_0^1(1-t)\mathcal H(t)\,\ud t,
 \\
 D_{\widetilde\psi_m}^{\rm sym}(\bmu_1,\bmu_0)
 =&\int_0^1\mathcal H(t)\,\ud t.
\end{align*}
Thus each directed divergence is at least one quarter of the symmetric
divergence, proving \Cref{eq:filter-directed-comparability}.

By \Cref{eq:power-radial-comparability},
$\lambda_m(\btheta+t\be)\leq2\lambda_m(\btheta+t'\be)$ for every
$t,t'\in[0,1]$.
Integrating
\Cref{eq:power-dual-score} therefore gives
\begin{align*}
 \inner{\be}{\by_m(\btheta+\be)-\by_m(\btheta)}
 \leq20\sqrt2\,\widehat\lambda^{3/2}
 \nbr\be_\infty^2,
\end{align*}
which is stronger than \Cref{eq:lifted-local-smoothness}.

Finally, write
$\bmu_\ell=(\lambda_\ell,\lambda_\ell\bx_\ell)$ for
$\ell\in\cbr{0,1}$.  Direct differentiation gives
\begin{align*}
 \bnu(t)
 =(1-t)\bmu_0+t\bmu_1
 = (\lambda(t),\lambda(t)\bx(t)),
 \qquad t\in[0,1],
 \qquad
 \dot\bx(t)
 =\frac{\lambda_0\lambda_1}{\lambda(t)^2}
 \rbr{\bx_1-\bx_0}.
\end{align*}
Since $\lambda_\ell\leq2\lambda_{1-\ell}$ for $\ell\in\cbr{0,1}$,
$\lambda(t)\geq\widehat\lambda/2$ and
$\lambda_0\lambda_1/\lambda(t)^2\geq1/2$.  Hence
\Cref{eq:power-coercivity} and weighted Cauchy--Schwarz give
\begin{align*}
 \dot\bnu(t)^\top\nabla^2\widetilde\psi_m(\bnu(t))\dot\bnu(t)
 \geq\frac1{8\sqrt2}\sqrt{\widehat\lambda}
 \nbr{\bx_1-\bx_0}_1^2.
\end{align*}
Taylor's integral formula in both orientations yields
\begin{align*}
 \min\cbr{
 D_{\widetilde\psi_m}(\bmu_1,\bmu_0),
 D_{\widetilde\psi_m}(\bmu_0,\bmu_1)}
 \geq\frac1{16\sqrt2}\sqrt{\widehat\lambda}
 \nbr{\bx_1-\bx_0}_1^2.
\end{align*}
Since $\Xi_m\geq1$, this proves \Cref{eq:lifted-local-norm-bound}.
\end{proof}

\section{EMA estimates for multilinear games}
\label{sec:polynomial-closure}

Use $\rho$ and $\gamma$ from
\Cref{eq:stable-filter-law,eq:normalized-learning-rate}.
\subsection{EMA kernels and weighted convolution}

\restate{lem:filter-kernel-bounds}

\begin{proof}
The transfer functions are
\begin{align*}
 \mathsf A(\varsigma)=\frac1{1-\rho\varsigma},
 \qquad
 \delta(\varsigma)=\frac{1-\varsigma}{1-\rho\varsigma},
 \qquad
 \mathsf K_h(\varsigma)=
 \frac{(1-\varsigma)^{h-1}}{(1-\rho\varsigma)^h}.
\end{align*}
Thus $\delta=\mathsf A\mathsf D$ and
$\mathsf K_h\mathsf D=\delta^h$.  Since these operators commute,
\begin{align*}
 \mathsf K_h(I-\mathsf D\delta^N)
 =\mathsf A\delta^{h-1}-\delta^{N+h}
 =\rbr{\mathsf A-\delta^{N+1}}\delta^{h-1}
 =\mathsf K_\star\delta^{h-1}.
\end{align*}

Set $R=1+1/(4N)$.  Then $\rho R<1$ and
\begin{align*}
 \rho R^2
 =1-\frac1{2N}-\frac7{16N^2}-\frac1{16N^3}<1.
\end{align*}
For $\vartheta\in\RR$ and $\varsigma=Re^{\mathrm i\vartheta}$,
\begin{align*}
 \abr{\delta(\varsigma)}^2
 =&\frac{1+R^2-2R\cos\vartheta}
 {1+\rho^2R^2-2\rho R\cos\vartheta},\\
 \frac{\ud}{\ud(\cos\vartheta)}\abr{\delta(\varsigma)}^2
 =&\frac{2R(1-\rho)(\rho R^2-1)}
 {\rbr{1+\rho^2R^2-2\rho R\cos\vartheta}^2}\leq0.
\end{align*}
The maximum on $\abr{\varsigma}=R$ is therefore attained at
$\varsigma=-R$.  Since
$R(4-3\rho)\leq3$ for $N\geq2$,
\begin{align*}
 \sup_{\abr{\varsigma}=R}\abr{\delta(\varsigma)}
 =\frac{1+R}{1+\rho R}
 =1+\frac{(1-\rho)R}{1+\rho R}
 \leq1+\frac3{4N}.
\end{align*}
For $h\leq N+1\leq3N/2$, this gives
\begin{align*}
 \sup_{\abr{\varsigma}=R}\abr{\delta(\varsigma)}^h
 \leq e^{9/8}<4.
\end{align*}
Moreover, $1-\rho R\geq3/(4N)$, and hence
\begin{align*}
 \sup_{\abr{\varsigma}=R}\abr{\mathsf K_h(\varsigma)}
 \leq&\frac{4}{1-\rho R}\leq6N,\\
 \sup_{\abr{\varsigma}=R}\abr{\mathsf K_\star(\varsigma)}
 \leq&\frac1{1-\rho R}+4\leq6N.
\end{align*}
The same upper bound $6N$ applies to $\delta^h$.  Cauchy's coefficient
estimate therefore gives
\begin{align*}
 \abr{k_s}\leq6NR^{-s},
 \qquad s\geq0.
\end{align*}

Fix $0\leq\epsilon\leq1/(8N)$.  The inequality
$\log(1+y)\geq y/(1+y)$ for $y\geq0$ gives
\begin{align*}
 \log R-\frac\epsilon2
 \geq\frac1{4N+1}-\frac1{16N}
 \geq\frac1{8N}.
\end{align*}
Thus $e^{\epsilon/2}/R\leq e^{-1/(8N)}$.  Since
$1-e^{-y}\geq y/2$ for $0\leq y\leq1$, we have
$1-e^{\epsilon/2}/R\geq1/(16N)$.  Consequently,
\begin{align*}
 \sum_{s\geq0}e^{\epsilon s/2}(s+1)^4\abr{k_s}
 \leq&6N\sum_{s\geq0}(s+1)^4
 \rbr{\frac{e^{\epsilon/2}}R}^s\\
 \leq&\frac{144N}{\rbr{1-e^{\epsilon/2}/R}^5}
 \leq144\cdot16^5N^6
 \leq 2^{28}N^6,
\end{align*}
where, for every $q\in[0,1)$,
$\sum_{s\geq0}(s+1)^4q^s
=(1+11q+11q^2+q^3)/(1-q)^5$.
The absolute convergence permits evaluation at $\varsigma=1$, so the
coefficients of $\delta^h$ sum to $\delta(1)^h=0$.

For the last claim, the output difference at time $t\geq1$ is
\begin{align*}
 \boldsymbol\tau^{(t)}=\boldsymbol\varphi
 \sum_{s\geq t}k_s.
\end{align*}
Apply \Cref{eq:filter-fourth-moment} with
$\epsilon=1/(8N)$.  For every $t\geq1$,
\begin{align*}
 \nbr{\boldsymbol\tau^{(t)}}
 \leq e^{-t/(16N)}\nbr{\boldsymbol\varphi}
 \sum_{s\geq0}e^{s/(16N)}\abr{k_s}
 \leq 2^{28}N^6e^{-t/(16N)}
 \nbr{\boldsymbol\varphi}.
\end{align*}
Since $e^{1/(8N)}-1\geq1/(8N)$, summing the squares gives
\begin{align*}
 \sum_{t=1}^{\infty}\nbr{\boldsymbol\tau^{(t)}}^2
 \leq\frac{2^{56}N^{12}}
 {e^{1/(8N)}-1}\nbr{\boldsymbol\varphi}^2
 \leq 2^{59}N^{13}\nbr{\boldsymbol\varphi}^2.
\end{align*}
\end{proof}

\restate{lem:slow-weight-convolution}

\begin{proof}
The hypothesis on $w$ gives
\begin{align*}
 \sqrt{w^{(t)}}\nbr{\rbr{\mathsf K\bz}^{(t)}}
 \leq\sum_{s\geq0}e^{\epsilon s/2}\abr{k_s}
 \sqrt{w^{(t-s)}}\nbr{\bz^{(t-s)}}.
\end{align*}
Extend the scalar sequence on the right by zero outside $[T]$.  The
$\ell_1*\ell_2\to\ell_2$ Young inequality gives the conclusion.
\end{proof}

\subsection{Score and response movement}
\label{sec:response-secants-locality}

\restate{lem:filter-score-locality}

\begin{proof}
For every $i\in[N]$ and $t\geq1$,
$\nbr{\bg_i^{(t)}}_\infty\leq1$, and
$\bg_i^{(s)}=\boldsymbol 0$ for every integer $s\leq0$.  Since
$\be_i=\delta^N\bg_i$,
\Cref{lem:filter-kernel-bounds} gives the first bound.  Moreover,
\begin{align*}
 \nbr{\rbr{(I-\mathsf D\delta^N)\bg_i}^{(t)}}_\infty
 =\nbr{\widehat\bg_i^{(t)}+\be_i^{(t-1)}}_\infty
 \leq1+2^{29}N^6
 \leq3\cdot2^{28}N^6.
\end{align*}
Fix $i\in[N]$ and $t\geq1$.  Then
\begin{align*}
 \eta^{-1}\rbr{\btheta_i^{(t+1)}-\btheta_i^{(t)}}
 =&\sum_{s=1}^t\bg_i^{(s)}+\widehat\bg_i^{(t+1)}
 -\sum_{s=1}^{t-1}\bg_i^{(s)}-\widehat\bg_i^{(t)}
 \\
 =&\rbr{(I-\mathsf D\delta^N)\bg_i}^{(t+1)}\\
 =&\widehat\bg_i^{(t+1)}+\be_i^{(t)}.
\end{align*}
Applying the preceding bound at time $t+1$ and multiplying by $\eta\Xi_i$
verifies the hypothesis of
\Cref{lem:square-root-entropy-geometry}.  Applying
\Cref{eq:filter-mass-stability} proves \Cref{eq:played-mass-drift}.
\end{proof}

\restate{lem:response-secant-estimates}

\begin{proof}
Fix $i\in[N]$ and $t\geq1$.  For every $\alpha\in[0,1]$,
\begin{align*}
 \btheta_i^{(t-1)}+\alpha\eta\bd_i^{(t)}-\btheta_i^{(t)}
 =-(1-\alpha)\eta\bd_i^{(t)},
\end{align*}
and
\Cref{eq:filter-pointwise-score-bound,eq:normalized-learning-rate} give
\begin{align*}
 \Xi_i\nbr{(1-\alpha)\eta\bd_i^{(t)}}_\infty
 \leq(1-\alpha)\,3\cdot2^{28}\gamma N^6
 \leq\frac1{100}.
\end{align*}
Therefore, \Cref{lem:square-root-entropy-geometry,eq:filter-mass-stability}
gives, for every $\alpha\in[0,1]$,
\begin{align*}
 \lambda_{m_i}\rbr{\btheta_i^{(t-1)}+\alpha\eta\bd_i^{(t)}}
 \leq e^{1/100}\lambda_i^{(t)}
 \leq2\lambda_i^{(t)}.
\end{align*}
Integrating \Cref{eq:filter-jacobian-size} over $\alpha\in[0,1]$ and using
$\eta\Xi_i\leq\gamma$ proves \Cref{eq:scaled-secant-size}.

Since $\btheta_i^{(t)}-\btheta_i^{(t-1)}=\eta\bd_i^{(t)}$,
\Cref{eq:lifted-local-norm-bound} gives
\begin{align*}
 \sqrt{\lambda_i^{(t)}}
 \nbr{\rbr{\mathsf D\bx_i}^{(t)}}_1^2
 \leq16\sqrt2\Xi_i
 D_{\widetilde\psi_i}\rbr{\bmu_i^{(t)},\bmu_i^{(t-1)}}.
\end{align*}
This is \Cref{eq:movement-path-bound}.  The identity in
\Cref{eq:direct-action-path}, \Cref{eq:scaled-secant-size}, and
\Cref{eq:filter-pointwise-score-bound} give
\begin{align*}
 \nbr{\rbr{\mathsf D\bx_i}^{(t)}}_1
 \leq3\sqrt2\cdot2^{30}\eta\Xi_iN^6
 \sqrt{\lambda_i^{(t)}},
\end{align*}
which proves \Cref{eq:pointwise-action-movement}.

Since $t$ was arbitrary, this proves the first three assertions.  It remains
to prove \Cref{eq:consecutive-secant-variation}.  Fix $t\geq1$.  The
fundamental theorem of calculus gives
\begin{align*}
 \overline{\bJ}_i^{(t)}-\eta\bJ_{m_i}^x\rbr{\btheta_i^{(t-1)}}
 =&\eta\int_0^1(1-\alpha)
 \frac{\ud}{\ud\alpha}\bJ_{m_i}^x\rbr{
 \btheta_i^{(t-1)}+\alpha\eta\bd_i^{(t)}}\ud\alpha,\\
 \eta\bJ_{m_i}^x\rbr{\btheta_i^{(t)}}-\overline{\bJ}_i^{(t)}
 =&\eta\int_0^1\alpha
 \frac{\ud}{\ud\alpha}\bJ_{m_i}^x\rbr{
 \btheta_i^{(t-1)}+\alpha\eta\bd_i^{(t)}}\ud\alpha.
\end{align*}
By \Cref{eq:filter-jacobian-variation}, the Cauchy--Schwarz inequality,
and \Cref{eq:filter-directed-comparability},
\begin{align*}
 \max\cbr{
 \nbr{\overline{\bJ}_i^{(t)}-
 \eta\bJ_{m_i}^x\rbr{\btheta_i^{(t-1)}}}_{\infty\to1}^2,
 \nbr{\overline{\bJ}_i^{(t)}-
 \eta\bJ_{m_i}^x\rbr{\btheta_i^{(t)}}}_{\infty\to1}^2}
 \leq&\eta^2\Xi_i^2D_{\widetilde\psi_i}^{\rm sym}
 \rbr{\bmu_i^{(t)},\bmu_i^{(t-1)}}\\
 \leq&4\eta^2\Xi_i^2D_{\widetilde\psi_i}
 \rbr{\bmu_i^{(t)},\bmu_i^{(t-1)}}.
\end{align*}
Since $t$ was arbitrary, the preceding bound holds for every $t\geq1$.
Since $\overline{\bJ}_i^{(1)}=\overline{\bJ}_i^{(0)}
=\eta\bJ_{m_i}^x(\boldsymbol 0)$, for every
$t\geq2$ the preceding bound gives
\begin{align*}
 \nbr{\rbr{\mathsf D\overline{\bJ}_i}^{(t)}}_{\infty\to1}^2
 \leq8\eta^2\Xi_i^2\sbr{
 D_{\widetilde\psi_i}\rbr{\bmu_i^{(t)},\bmu_i^{(t-1)}}
 +D_{\widetilde\psi_i}\rbr{\bmu_i^{(t-1)},\bmu_i^{(t-2)}}}.
\end{align*}
Fix $T\geq1$.  Summing over $t\in[T]$ proves
\Cref{eq:consecutive-secant-variation}.
\end{proof}

\restate{lem:filtered-action-path-bound}

\begin{proof}
Extend $\lambda_j$ constantly to all nonpositive integer times.  By
\Cref{eq:played-mass-drift}, for $t\in[T]$ and $s\geq0$,
\begin{align*}
 \abr{\log\frac{\sqrt{\lambda_j^{(t)}}}
 {\sqrt{\lambda_j^{(t-s)}}}}
 \leq3\cdot2^{27}\gamma N^6s.
\end{align*}
The rate condition gives
$3\cdot2^{27}\gamma N^6s\leq s/(8N)$.  Since
$\delta^h=\mathsf K_h\mathsf D$ and
$\mathsf D\bx_j$ vanishes at every nonpositive integer time,
\Cref{lem:filter-kernel-bounds,lem:slow-weight-convolution} gives
\begin{align*}
 \sum_{t=1}^T\sqrt{\lambda_j^{(t)}}
 \nbr{\rbr{\delta^h\bx_j}^{(t)}}_1^2
 \leq&2^{56}N^{12}
 \sum_{t=1}^T\sqrt{\lambda_j^{(t)}}
 \nbr{\rbr{\mathsf D\bx_j}^{(t)}}_1^2\\
 \leq&2^{60}\sqrt2\,N^{12}\Xi_j\mathcal P_{j,T},
\end{align*}
where the last inequality follows by summing
\Cref{eq:movement-path-bound} over $t\in[T]$.
\end{proof}

\restate{lem:filtered-response-identity}

\begin{proof}
Apply $\mathsf K_h$ to \Cref{eq:direct-action-path}, commute it past
$\overline{\bJ}_i$, and use
\Cref{eq:filter-factorization,eq:ema-controller-identity} to obtain
\Cref{eq:filtered-response-order-reduction}.

Write $k_s=\operatorname{ker}(\mathsf K_h)_s$.  By the definition of
$\bc_{i,h}$, for every $t\geq1$,
\begin{align*}
 \bc_{i,h}^{(t)}
 =&\sum_{s\geq1}k_s
 \rbr{\overline{\bJ}_i^{(t-s)}-
 \overline{\bJ}_i^{(t)}}\bd_i^{(t-s)}\\
 =&-\sum_{s\geq1}k_s\sum_{r=0}^{s-1}
 \rbr{\mathsf D\overline{\bJ}_i}^{(t-r)}\bd_i^{(t-s)}.
\end{align*}
For every $t\in[T]$, $s\geq1$, and integer $0\leq r<s$,
\Cref{eq:filter-pointwise-score-bound} gives
\begin{align*}
 \nbr{\rbr{\mathsf D\overline{\bJ}_i}^{(t-r)}
 \bd_i^{(t-s)}}_1
 \leq3\cdot2^{28}N^6
 \nbr{\rbr{\mathsf D\overline{\bJ}_i}^{(t-r)}}_{\infty\to1}.
\end{align*}
Weighted Cauchy--Schwarz over the pairs $(s,r)$, followed by the time
shifts $\alpha=t-r$, gives
\begin{align*}
 \sum_{t=1}^T
 \nbr{\bc_{i,h}^{(t)}}_1^2
 \leq&9\cdot2^{56}N^{12}
 \rbr{\sum_{s\geq1}s\abr{k_s}}
 \sum_{s\geq1}\abr{k_s}\sum_{r=0}^{s-1}
 \sum_{t=1}^T
 \nbr{\rbr{\mathsf D\overline{\bJ}_i}^{(t-r)}}_{\infty\to1}^2\\
 \leq&9\cdot2^{56}N^{12}
 \rbr{\sum_{s\geq1}s\abr{k_s}}^2
 \sum_{\alpha=1}^T
 \nbr{\rbr{\mathsf D\overline{\bJ}_i}^{(\alpha)}}_{\infty\to1}^2\\
 \leq&9\cdot2^{116}
 \eta^2\Xi_i^2N^{24}\mathcal P_{i,T}.
\end{align*}
The second inequality uses
$\overline{\bJ}_i^{(\tau)}=\eta\bJ_{m_i}^x(\boldsymbol0)$ for every integer
$\tau\leq0$, and the last uses
\Cref{eq:filter-fourth-moment,eq:consecutive-secant-variation}.
By the definition of $\bc_{i,h}$, the preceding display proves the first
inequality in \Cref{eq:filtered-response-residual-bound}.  The inequality
$\eta\Xi_i\leq\gamma$ gives
\begin{align*}
 \sum_{t=1}^T\nbr{\bc_{i,h}^{(t)}}_1^2
 \leq9\cdot2^{116}
 \gamma^2N^{24}\mathcal P_{i,T}.
\end{align*}
This proves the second inequality.
\end{proof}

\subsection{Filtered action and payoff differences}
\label{sec:filtered-multiaffine-proof}

\restate{lem:filtered-multiaffine-expansion}

\begin{proof}
For every $i\in[N]$, integer $1\leq q\leq N$, pairwise distinct
$j_1,\ldots,j_q\in[N]$, profile
$\bx\in\prod_{k=1}^N\Delta^{m_k}$, and directions
$\bz_{j_\ell}\in\RR^{m_{j_\ell}}$, $\ell\in[q]$, multilinearity and the
payoff range $[0,1]$ give
\begin{align*}
 \nbr{\partial_{j_1}\cdots\partial_{j_q}\mathcal G_i(\bx)
 [\bz_{j_1},\ldots,\bz_{j_q}]}_\infty
 \leq2\prod_{\ell=1}^q\nbr{\bz_{j_\ell}}_1.
\end{align*}
Indeed, for every output coordinate, expanding $\cU_i$ over pure
profiles bounds the absolute value of the mixed directional derivative of
each term in
$\mathcal G_i(\bx)=\cU_i(\bx_{-i})
-\inner{\cU_i(\bx_{-i})}{\bx_i}\one$ by
$\prod_{\ell=1}^q\nbr{\bz_{j_\ell}}_1$.  This proves the first inequality in
the statement by taking $q=1$.

Fix $i\in[N]$ and an integer $1\leq h\leq N$.  Use
$(\bx_j)_{j\in[N]}$ with the extension from \Cref{sec:finite-algorithm}, and set
$\widetilde\bg_i^{(t)}=\mathcal G_i(\bx^{(t)})$ for every integer $t$.  Write
$k_{h,s}=\operatorname{ker}(\delta^h)_s$ for $s\geq0$.  Since
$\sum_{s\geq0}k_{h,s}=0$, for every integer $t$ and $s\geq0$ define the
Taylor remainder by
\begin{align*}
 \boldsymbol\chi_{i,t,s}^{\rm Tay}
 \coloneqq\mathcal G_i\rbr{\bx^{(t-s)}}-\mathcal G_i\rbr{\bx^{(t)}}
 -\sum_{j=1}^N \partial_j\mathcal G_i\rbr{\bx^{(t)}}
 \rbr{\bx_j^{(t-s)}-\bx_j^{(t)}}.
\end{align*}
Taylor's formula gives
\begin{align*}
 \rbr{\delta^h\widetilde\bg_i}^{(t)}
 =\sum_{j=1}^N \partial_j\mathcal G_i\rbr{\bx^{(t)}}
 \rbr{\delta^h\bx_j}^{(t)}
 +\sum_{s\geq0}k_{h,s}\boldsymbol\chi_{i,t,s}^{\rm Tay}.
\end{align*}
For every integer $t$, nonnegative integer $s$, distinct $j,k\in[N]$, and
$\alpha\in[0,1]$, the profile
$\bx^{(t)}+\alpha\rbr{\bx^{(t-s)}-\bx^{(t)}}$ belongs to
$\prod_{\ell=1}^N\Delta^{m_\ell}$.  The bound with $q=2$ gives
\begin{align*}
 &\nbr{\partial_j\partial_k\mathcal G_i\rbr{
 \bx^{(t)}+\alpha\rbr{\bx^{(t-s)}-\bx^{(t)}}}
 \sbr{\bx_j^{(t-s)}-\bx_j^{(t)},
 \bx_k^{(t-s)}-\bx_k^{(t)}}}_\infty\\
 \leq&2\nbr{\bx_j^{(t-s)}-\bx_j^{(t)}}_1
 \nbr{\bx_k^{(t-s)}-\bx_k^{(t)}}_1.
\end{align*}
Since $\partial_j^2\mathcal G_i=0$, the integral Taylor remainder and
$\int_0^1 2(1-\alpha)\ud\alpha=1$ yield
\begin{align*}
 \nbr{\boldsymbol\chi_{i,t,s}^{\rm Tay}}_\infty
 \leq\rbr{\sum_{j=1}^N
 \nbr{\bx_j^{(t-s)}-\bx_j^{(t)}}_1}^2.
\end{align*}
Since $\bg_i^{(s)}=\boldsymbol0$ for every integer $s\leq0$, for every
$t\geq1$,
\begin{align*}
 \rbr{\delta^h\bg_i}^{(t)}
 =\rbr{\delta^h\widetilde\bg_i}^{(t)}
 -\mathcal G_i\rbr{\bx^{(0)}}\sum_{s\geq t}k_{h,s}.
\end{align*}
Therefore, for every $t\geq1$, set
\begin{align*}
 \boldsymbol\chi_{i,h}^{(t)}
 \coloneqq\sum_{s\geq0}k_{h,s}\boldsymbol\chi_{i,t,s}^{\rm Tay}
 -\mathcal G_i\rbr{\bx^{(0)}}\sum_{s\geq t}k_{h,s}.
\end{align*}

Fix $T\geq1$.  Weighted Cauchy--Schwarz and H\"older's inequality give
\begin{align*}
 \sum_{t=1}^T\nbr{\sum_{s\geq0}
 k_{h,s}\boldsymbol\chi_{i,t,s}^{\rm Tay}}_\infty^2
 \leq N^3\rbr{\sum_{s\geq0}\abr{k_{h,s}}}
 \sum_{s\geq0}\abr{k_{h,s}}
 \sum_{j=1}^N\sum_{t=1}^T
 \nbr{\bx_j^{(t-s)}-\bx_j^{(t)}}_1^4.
\end{align*}
For every $j\in[N]$, integer $t$, and integer $s\geq1$, telescoping and
H\"older's inequality yield
\begin{align*}
 \nbr{\bx_j^{(t-s)}-\bx_j^{(t)}}_1^4
 \leq s^3\sum_{r=0}^{s-1}
 \nbr{\rbr{\mathsf D\bx_j}^{(t-r)}}_1^4.
\end{align*}
For every $j\in[N]$ and integer $s\geq1$, the extension of $\bx_j$ from
\Cref{sec:finite-algorithm} gives
\begin{align*}
 \sum_{t=1}^T\sum_{r=0}^{s-1}
 \nbr{\rbr{\mathsf D\bx_j}^{(t-r)}}_1^4
 \leq s\sum_{t=1}^T
 \nbr{\rbr{\mathsf D\bx_j}^{(t)}}_1^4.
\end{align*}
Substitution and \Cref{eq:filter-fourth-moment} give
\begin{align*}
 \sum_{t=1}^T\nbr{\sum_{s\geq0}
 k_{h,s}\boldsymbol\chi_{i,t,s}^{\rm Tay}}_\infty^2
 \leq2^{56}N^{15}
 \sum_{j=1}^N\sum_{t=1}^T
 \nbr{\rbr{\mathsf D\bx_j}^{(t)}}_1^4.
\end{align*}
By \Cref{eq:pointwise-action-movement,eq:movement-path-bound},
\begin{align*}
 \sum_{j=1}^N\sum_{t=1}^T
 \nbr{\rbr{\mathsf D\bx_j}^{(t)}}_1^4
 \leq&9\cdot2^{61}\gamma^2N^{12}
 \sum_{j=1}^N\sum_{t=1}^T
 \lambda_j^{(t)}\nbr{\rbr{\mathsf D\bx_j}^{(t)}}_1^2\\
 \leq&9\sqrt2\cdot2^{65}
 \gamma^2N^{12}\mathcal P_T^\Xi.
\end{align*}
The second inequality uses $\lambda_j^{(t)}\leq
\sqrt{\lambda_j^{(t)}}$, followed by
\Cref{eq:movement-path-bound}.  Substitution gives
\begin{align*}
 \sum_{t=1}^T\nbr{\sum_{s\geq0}
 k_{h,s}\boldsymbol\chi_{i,t,s}^{\rm Tay}}_\infty^2
 \leq9\sqrt2\cdot2^{121}
 \gamma^2N^{27}\mathcal P_T^\Xi.
\end{align*}
Since $\nbr{\mathcal G_i(\bx^{(0)})}_\infty\leq1$,
\Cref{eq:filter-startup-tail} gives
\begin{align*}
 \sum_{t=1}^T\nbr{\mathcal G_i\rbr{\bx^{(0)}}
 \sum_{s\geq t}k_{h,s}}_\infty^2
 \leq2^{59}N^{13}.
\end{align*}
By the definition of $\boldsymbol\chi_{i,h}^{(t)}$, for every $t\geq1$,
\begin{align*}
 \nbr{\boldsymbol\chi_{i,h}^{(t)}}_\infty^2
 \leq2\nbr{\sum_{s\geq0}
 k_{h,s}\boldsymbol\chi_{i,t,s}^{\rm Tay}}_\infty^2
 +2\nbr{\mathcal G_i\rbr{\bx^{(0)}}
 \sum_{s\geq t}k_{h,s}}_\infty^2.
\end{align*}
Summing this inequality and using the preceding two estimates proves
\Cref{eq:filtered-multiaffine-residual-sharp}.  Finally,
$\gamma\leq N^{-20}$ implies
\begin{align*}
 \gamma^2N^{27}\leq N^{-13}\leq N^{13},
\end{align*}
which proves \Cref{eq:filtered-multiaffine-residual}.
\end{proof}

\end{document}